\documentclass[a4paper,12pt]{article}

\usepackage[utf8]{inputenc}
\usepackage{amsmath,amssymb,amsthm,amsfonts}
\usepackage{colortbl}
\usepackage[algo2e,ruled,vlined,linesnumbered]{algorithm2e}
\usepackage{xspace}
\usepackage{mathtools}
\usepackage{tikz}
\usepackage{url}

\allowdisplaybreaks[4]
\newcommand{\N}{\mathbb{N}}
\newcommand{\R}{\mathbb{R}}

\renewcommand{\epsilon}{\varepsilon}
\newcommand{\eps}{\epsilon}

\DeclareMathOperator{\Bin}{Bin}

\newcommand{\oea}{${(1 + 1)}$~EA\xspace}
\newcommand{\lea}{${(1 + \lambda)}$~EA\xspace}
\newcommand{\llea}{${(\lambda \overset{_{1:1}}{+} \lambda)}$~EA\xspace}
\newcommand{\mea}{${(\mu + 1)}$~EA\xspace}
\newcommand{\ea}{${(\mu + \lambda)}$~EA\xspace}
\newcommand{\commaea}{${(\mu, \lambda)}$~EA\xspace}

\newcommand{\onemax}{\textsc{OneMax}\xspace}

\newcommand{\leadingones}{\textsc{LeadingOnes}\xspace}

\newtheorem{theorem}{Theorem}
\newtheorem{lemma}{Lemma}

\newtheorem{corollary}{Corollary}

\begin{document}
{\sloppy

\title{A Tight Runtime Analysis for the \ea\thanks{A preliminary version of this work~\cite{AntipovDFH18} was presented at the \emph{Genetic and Evolutionary Computation Conference} (GECCO) 2018. In this version, the presentation was improved by rewriting almost the entire text, by giving a clearer comparison with the previous state of the art, by making many proofs more rigorous, by extending the lower bounds to arbitrary fitness functions (subject to a mild restriction on the number of global optima), and by extending our results to the so-called $(N+N)$ EA using a fair parent selection.}}

\author{Denis Antipov \\
ITMO University\\
Saint-Petersburg\\
Russia
\and
Benjamin Doerr\\ Laboratoire d'Informatique (LIX)\\ CNRS, \'Ecole Polytechnique\\ Institute Polytechnique de Paris\\ Palaiseau, France}

\maketitle

\begin{abstract}
Despite significant progress in the theory of evolutionary algorithms,
the theoretical understanding of evolutionary algorithms which use non-trivial 
populations remains challenging and only few rigorous results exist.
Already for the most basic problem, the determination of the asymptotic
runtime of the $(\mu+\lambda)$ evolutionary algorithm on the simple
\onemax benchmark function, only the special cases $\mu=1$ and
$\lambda=1$ have been solved.

In this work, we analyze this long-standing problem and show the
asymptotically tight result that the runtime $T$, the number of iterations until the optimum is found, satisfies \[E[T] =
\Theta\bigg(\frac{n\log n}{\lambda}+\frac{n}{\lambda / \mu} +
\frac{n\log^+\log^+ (\lambda/ \mu)}{\log^+ (\lambda / \mu)}\bigg),\] where $\log^+
x := \max\{1, \log x\}$ for all $x > 0$.

The same methods allow to improve the previous-best $O(\frac{n \log n}{\lambda} + n \log \lambda)$ runtime guarantee for the $(\lambda+\lambda)$~EA with fair parent selection to a tight $\Theta(\frac{n \log n}{\lambda} + n)$ runtime result.
\end{abstract}

\section{Introduction}

Evolutionary algorithms are general-purpose optimization heuristics and have been successfully
applied to a broad range of computational problems. While the majority of the research in evolutionary computation is applied and experimental, the last decades have seen a growing number of theoretical analyses of evolutionary algorithms. Due to the difficult nature of the stochastic processes describing the runs of evolutionary algorithms, the vast majority of these works regards very simple algorithms like the \oea, which has both a parent population and an offspring population of size one. Such works, while innocent looking in their problem statement, can be surprisingly challenging from the mathematical point of view, see, e.g., the long series of works on how the \oea optimizes pseudo-Boolean linear functions which started with the seminal paper~\cite{DrosteJW02}. Also, while these example problems are far from the real applications, many of the theoretical works have contributed to the understanding of the working principles of evolutionary algorithms (see, e.g,~\cite{Droste05,DoerrHK12,DangFKKLOSS18}), have given advice on how to set parameters and take other design choices (see, e.g.,~\cite{BottcherDN10,Witt13,RoweD14}), and even have proposed new algorithms (see, e.g,~\cite{DoerrLMN17,DoerrDE15}).

Still, it remains dissatisfying that there are only relatively few works on true population-based algorithms as this bears the risk that we do not really understand the role of populations in evolutionary computation. What is clearly true, and the reason for the lack of such works, is that the stochastic processes become much more complicated when non-trivial populations come into play.

To make some progress towards a better understanding of population-based algorithms, we regard the most simple population-based problem, namely how the elitist \ea optimizes the \onemax benchmark problem (see Section~\ref{sec:problem_statemnt} for the details of this problem). With the corresponding problem for the \lea mostly solved in 2005 \cite{JansenJW05} (see~\cite[Section~8]{DoerrK15} for a more complete picture) and the problem for the \mea solved in 2006 \cite{Witt06}, it is fair to call this a long-standing open problem. In the conclusion of his paper~\cite{Witt06}, Witt writes ``the most interesting direction seems to be an extension to $(\mu+\lambda)$ strategies
by a combination with the existing theory on the $(1+\lambda)$ EA.''

\subsection{Our Results}

We give a complete answer to this question and prove that for arbitrary values of $\mu$ and $\lambda$ (which can be functions of the problem size $n$, however, for the lower bound we assume that $\mu$ is at most polynomial in $n$), the expected number of iterations the \ea takes to find the optimum of the \onemax function is
\[E[T] =
\Theta\bigg(\frac{n\log n}{\lambda}+\frac{n}{\lambda / \mu} +
\frac{n\log^+\log^+ (\lambda/ \mu)}{\log^+ (\lambda / \mu)}\bigg),\] where $\log^+
x := \max\{1, \log x\}$ for all $x > 0$.
This result subsumes the previous results for the \lea and \mea obtained in~\cite{JansenJW05,DoerrK15,Witt06}.

This runtime guarantee shows, e.g., that using a true parent population of size at most $\max\{\log n, \lambda\}$ does not reduce the asymptotic runtime compared to $\mu=1$. Such information can be useful since it is known that larger parent population sizes can increase the robustness to noise, see,~e.g.,~\cite{GiessenK16}.

With our methods, we can also analyze a related algorithm. He and Yao~\cite{HeY04} and Chen, He, Sun, Chen, and Yao~\cite{Chen09} analyzed a version of the \ea in which $\mu = \lambda$ and each parent produces exactly one offspring. We shall call this algorithm the \llea for brevity.
We prove a tight runtime bound of $\Theta(\frac{n \log n}{\lambda} + n)$ iterations, which also shows that the fairness in the parent selection does not change the asymptotic runtime in this problem.

To prove our bounds, we in particular build on Witt's~\cite{Witt03,Witt06,Witt08} family tree argument. The main idea of this argument is to consider a tree graph the vertices of which are the individuals created during the evolution process and each path from the root to a vertex corresponds to the series of mutations which led to the creation of this vertex. Selection does not play any role in this structure, so when working with family trees we usually assume that all individuals with a corresponding vertex in the tree can potentially be present in the current population. Different from Witt's approach (and different from all other works using his method that we are aware of), we work in a complete tree that contains all possible family trees and in this structure argue about which individuals really exist and whether they are an optimal solution. It appears to us that this approach is technically easier than the previous approaches, which first argue that with high probability the true family tree has a certain structure (e.g., a small height) and then, conditioning on this, argue that within such a restricted structure an optimal solution is hard to reach. We also believe that our approach facilitates the uniform analysis of trees having different structures (as, e.g., in our work, in which different relative sizes of $\mu$ and $\lambda$ can lead to very different characteristics of the tree).

Using this argument that does not regard the selection process we have obtained the lower bound that holds not only for the \onemax function, but for any 
function with a unique optimum in the same way as it was done for the \mea in~\cite{Witt06}. Our arguments let us extend the lower bounds to the functions with multiple optima (however, the number of the optima in these functions must be restricted). In the case of \mea this extension holds for a broader class of functions than the similar extension in~\cite{Witt06} which works only for the functions with a unique optimum.

\subsection{Previous Works}

The field of mathematical runtime analysis of evolutionary algorithms aims at increasing our understanding via proven results on the performance of evolutionary algorithms. Due to the difficulty of mathematical understanding of complicated population dynamics, the large majority of works in this field considers algorithms with trivial populations. These algorithms may seem trivial, however already allow deep results like the proof of the $O(n \log n)$ expected runtime of the $(1+1)$ EA on all linear pseudo-Boolean functions~\cite{DrosteJW02,DoerrG13algo}. They give surprising insights like the fact that monotonic functions can be difficult for simple EAs~\cite{DoerrJSWZ13,ColinDF14,Lengler18}, and have spurred the development of many useful analysis methods~\cite{HeY01,DoerrJW12}.

Despite the mathematical challenges, some results exist on algorithms using non-trivial populations. While such results are quite rare, due the growth of the field in the last 20 years they are still too numerous to be described here exhaustively. Therefore, we describe in the following those results which regard our research problem or special cases of it as well as a few related results.

The two obvious special cases of our problem are runtime analysis of the \lea and the \mea on \onemax. In~\cite{DoerrK15}, the runtime of the \lea on the class of linear functions is analyzed, which contains the \onemax function. A tight bound of $\Theta(\frac{n\log n}{\lambda} + \frac{n\log^+\log^+\lambda}{\log^+\lambda})$ is proven for the expected runtime (number of iterations until the optimum is found) of the \lea maximizing the \onemax function.
This extends the earlier result~\cite{JansenJW05}, which shows this bound for $\lambda = O(\frac{\log(n) \log\log(n)}{\log\log\log(n)})$, note that in this case the bound simplifies to $\Theta(\frac{n \log(n)}{\lambda})$, and which shows further that for asymptotically larger values of $\lambda$, the expected runtime is $\omega(\frac{n \log(n)}{\lambda})$.

Witt \cite{Witt06} studied the $(\mu+1)$ EA on the three
pseudo-Boolean functions \leadingones, \onemax and SPC. 
For the \onemax problem, under the mild assumption that $\mu$ is polynomially bounded in $n$, he proved that the expected runtime of the $(\mu + 1)$ EA is $\Theta(\mu n + n\log n).$

For algorithms with non-trivial parent and offspring population sizes, the following is known. The only work
regarding the classic \ea for general $\mu$ and $\lambda$ is~\cite{Qian16}. Using the recent switch analysis technique~\cite{YuQZ15} and assuming that $\mu$ and $\lambda$ are polynomially bounded in~$n$, it was shown that the \ea needs an expected number of
\[
  \Omega\left(\frac{n\log n}{\lambda} + \frac{\mu}{\lambda} + \frac{n\log\log n}{\log n}\right)
\]
iterations to find the optimum of any function $f : \{0,1\}^n \to \R$ with unique optimum. This bound is of smaller asymptotic order (and thus weaker) than ours when $\mu = \omega(\log n)$ and $\frac{\lambda}{\mu} < e^e$ or when $\log\frac\lambda\mu =\omega(\log n)$, see the discussion at the end of Section~\ref{scn:lower}.

For the \llea the first result~\cite[Theorem~4]{HeY04} considers the runtime on the \onemax in a special case when $\lambda = n$. It shows an upper bound of $O(n)$ iterations,which is tight as shown by our lower bound.

For the \llea with general $\lambda$, Chen et al.~\cite[Proposition~4]{Chen09} show an optimization time of $O(\frac{n \log n}{\lambda} + n \log \lambda)$ iterations. They conjecture a runtime of $O(\frac{n \log n}{\lambda} + n \log\log n)$~\cite[Conjecture~3]{Chen09}, which is asymptotically at least as good and which is stronger for $\lambda=\omega(\log n)$.
Our result improves over this bound and the conjecture for $\lambda=\omega(\frac{\log n}{\log\log n})$ as discussed in Section~\ref{sec:llea}.

We also find notable the result of Dang and Lehre~\cite{DangL16} (see~\cite{AntipovDY19} for recent small improvements) where they proved the upper bound for the \commaea on the \onemax of $O(n\lambda\log\lambda)$ fitness evaluations when $\lambda > \mu e$ and $\lambda = \Omega(\log(n))$. This runtime can be seen as the upper bound for the \ea, since the population of the \ea is always better than the population of \commaea after the same number of iterations in the dominating sense.
In Section~\ref{sec:comparison-upper} we prove that in the parameters setting regarded in~\cite{DangL16} our upper bound is asymptotically smaller.

\subsection{Organization of the Work}

The remainder of the paper is organized as follows.  Section~\ref{sc:preliminaries} gives a formal description of the \ea and introduces the notation that we use in the paper. In Section~\ref{sc:ub} we prove an upper bound of $O(\frac{n\log n}{\lambda}+\frac{n\mu}{\lambda}+n)$ for the general case and a tighter bound of
 $O(\frac{n\log\log \frac{\lambda}{\mu}}{\log\frac{\lambda}{\mu}} + \frac{n\log n}{\lambda} )$ for the case of $\frac{\lambda}{\mu} > e^e$, when the algorithm is able to gain more than one fitness level during the major part of the optimization process. Section~\ref{scn:lower} introduces the notion of \emph{complete trees} and proves a lower bound matching our upper bounds. In Section~\ref{scn:lower-extended} we extend our lower bounds to a much broader class of functions than just \onemax. In Section~\ref{sec:llea} we provide an analysis of the \llea, for which our results cannot be applied directly. The paper ends with a short conclusion and ideas for future work in Section~\ref{scn:discussions}.

\section{Preliminaries}
\label{sc:preliminaries}

\subsection{Notation}
\label{sec:notation}

In this subsection we shortly overview the notation used in this paper in order to avoid misunderstanding raised by the plenty of other notations used in the mathematical world nowadays. By $\N$ we denote the set of positive integer numbers and by $\N_0$ we denote the set of non-negative integer numbers. By $[a..b]$ with $a, b \in \N$ we denote an integer interval which includes its borders. By $[a, b]$ and $(a, b)$ with $a, b \in \R \cup \{-\infty, +\infty\}$ we denote a real-valued interval including and excluding its borders respectively. We denote the binomial distribution with parameters $n$ and $p$ by $\Bin(n, p)$. If some random variable $X$ follows some distribution $D$, we write $X \sim D$. We use a multiplicative notation of the binomial coefficient, that is
\[
  \binom{n}{k} = \frac{n(n - 1)\dots(n - (k - 1))}{k!}.
\]
This implies that the binomial coefficients are also defined for $n < k$ and in this case $\binom{n}{k} = 0$.

\subsection{Problem Statement}
\label{sec:problem_statemnt}

In this section, we provide the definitions necessary to formalize the problem we analyze in this work.
Our study focuses on evolutionary algorithms that aim at optimizing pseudo-Boolean functions, that is, functions of the form $f:\{0, 1\}^n\to \R$.

The \ea formulated as Algorithm~\ref{alg:ea} is a simple mutation-based elitist evolutionary algorithm. In each iteration of the algorithm, we independently generate $\lambda$ offspring each by selecting an individual from the parent population uniformly at random and mutating it.
We use standard-bit mutation with the standard mutation rate $p = \frac 1n$, that is, we flip each bit independently with probability $\frac 1n$. We note without proof that our results hold as well for any other mutation rate $p = c/n$, where $c$ is a constant. 

\begin{algorithm2e}[t]%

	\underline{\textbf{Initialization:}}\\
	Create a population of $\mu$ individuals by choosing $x^{(i)} \in \{0,1\}^n$, $1\leq i \leq \mu$ uniformly at random. Let the multiset $X^{(0)} := \{x^{(1)}, ..., x^{(\mu)}\}$ be the population at time 0. Let $t := 0$.

 \underline{\textbf{Optimization:}}\\
\While{an optimum has not been reached}{
$X' := X^{(t)}$\;
{\textbf{Mutation phase:}}\\
\For{$i=1, \ldots, \lambda$\label{line:mutstart}}{
Choose $x \in X^{(t)}$ uniformly at random\;
Create $x'$ by
flipping each bit of $x$ with probability $p$\;
$X' := X' \cup \{x'\}$\;
}
{\textbf{Selection phase:}}

Create the multiset $X^{(t+1)}$, the population at time $t+1$, by deleting the $\lambda$ individuals with lowest
$f$-value in $X'$\;
$t:=t+1$\;
}

\caption{The \ea, maximizing a given function $f : \{0,1\}^n \to \R$, with population size $\mu$, offspring population size~$\lambda$ and mutation rate $p$. We shall exclusively regard the mutation rate $p = \frac 1n$. We did not specify how to break ties in the selection phase since our results are valid for any tie-breaking rule. Usually, one would prefer offspring over parents and break the remaining ties randomly.}
\label{alg:ea}
\end{algorithm2e}

As objective function $f$, also called \emph{fitness function},  we consider the classic \onemax function, which was the starting point for many theoretical investigations in this field. This function $\onemax : \{0,1\}^n \to \mathbb{R}$ is defined by $\onemax(x) = \sum_{i = 1}^n x_i$ for all $x \in \{0,1\}^n$. In other words, \onemax returns the number of one-bits in its argument. Without proof we note that due to the unbiasedness of the operators used by all considered algorithms all our results also hold for the so-called \emph{generalized} \onemax function, denoted by $\onemax_z$. This function has some hidden bit-string $z$ and returns the number of coinciding bits in its argument and $z$. In other words,
\[
  \onemax_z(x) = \sum_{i = 1}^n (1 - |z_i - x_i|) = n - H(x, z),
\]
where $H(x, z)$ stands for the Hamming distance.

\subsection{Useful Tools}
\label{sec:tools}

A central argument in our analysis is the following Markov chain 
argument similar to the classic fitness levels technique of Wegener~\cite{Wegener01}.

\begin{theorem}\label{th:levels}
Let the space $S$ of all possible populations of some population-based algorithm be divided into $m$ disjoint sets $A_1, \dots, A_m$ that are called \emph{levels}. We write $A_{\ge i} = \bigcup_{j = i}^{m} A_j$ for all $i \in [1..m]$. 

Let $P_t$ be the population of the algorithm after iteration $t$. Assume that for all $t \ge 0$ and ${i \in [2..m]}$, we have that $P_t \in A_{i}$ implies $\Pr[P_{t + 1} \in A_{\ge i}] = 1$. Let $T$ be the minimum number $t$ such that $P_t \in A_m$. 
\begin{enumerate}
\item Assume that there are $T_1, \dots, T_{m-1} \ge 0$ such that for all $t \ge 0$ and ${i \in [1..m-1]}$ we have that if $P_t \in A_i$, then $E[\min\{s \mid s \in \N, P_{t + s} \in A_{\ge i+1}\}] \le T_i$ (for all possible $P_0, \dots, P_{t - 1}$). Then \[E[T] \le \sum_{i=1}^{m-1} T_i.\]
\item Assume that there are $p_1, p_2, \dots, p_{m-1}$ such that for all $t \ge 0$ and ${i \in [1..m-1]}$ we have that if $P_t \in A_i$, then $\Pr[P_{t + 1} \in A_{\ge i+1}] \ge p_i$ (for all possible $P_0, \dots, P_{t - 1}$). Then \[E[T] \le \sum_{i = 1}^{m - 1} \frac 1 {p_i}.\]
\end{enumerate}
\end{theorem}

The proof is standard, but for the reason of completeness we quickly state it.

\begin{proof}
  We start by proving the first claim. Consider a run of the algorithm. Let $t_i = \min\{t \mid P_t \in A_{\ge i}\}$. Let $i \in [1..m-1]$. We analyze the random variable $t_{i+1} - t_i$. If there is no $t$ with $P_t \in A_i$, then $t_i = t_{i+1}$ simply by the definition of the $t_i$. Otherwise, by our assumptions, we have $E[t_{i+1}-t_i] \le T_i$. Note that this applies trivially also to the first case where we just saw that $t_{i+1} - t_i = 0$. Hence, from $T = t_m = \sum_{i = 1}^{m-1} (t_{i+1} - t_i)$ we conclude 
  \[E[T] = \sum_{i = 1}^{m-1} E[t_{i+1} - t_i] \le \sum_{i=1}^{m-1} T_i.\]

  To prove the second claim, we note that by our assumptions $t_{i+1} - t_i$ is stochastically dominated by a geometric distribution with success rate $p_i$. Hence $E[t_{i+1} - t_i] \le \frac 1 {p_i}$, and the claim follows as above.
\end{proof}

To ease the presentation, we use the following language. We say that \emph{the algorithm is on level $i$} if the current population is in the level $A_i$. We also say that \emph{the algorithm gains a level} or \emph{the algorithm leaves the current level} if the new population is at the higher level than the previous one.

In our proofs we shall use the following result for random variables with binomial distribution from~\cite{Greenberg14}. An elementary proof for it was given in~\cite{Doerr18exceedexp}.

\begin{lemma}
\label{lm:multbin}
Let $X \sim \Bin(n, p)$ such that $p > 1/n$. Then $\Pr(X \ge E[X]) > 1/4$.
\end{lemma}

We also use frequently the following inequality in our proofs, so we formulate it as a separate lemma.
\begin{lemma}\label{lm:bernoulli}
For any $x \in (0, 1]$ and any $n > 0$ we have
\[
1 - (1 - x)^n \ge \frac{1}{1 + \frac{1}{xn}}.
\]
\end{lemma}
As was pointed out by one of the reviewers, this lemma is a special case of Lemma 31 in~\cite{DangL16}, which states that for all $n \in \N$ and $x \ge 0$ we have $1 - (1 - x)^n \ge 1 - e^{-xn} \ge \frac{xn}{1 + xn}$, except we do not have the constraint that $n$ is an integer. Although the proof of~\cite[Lemma 31]{DangL16} is true for the case $x \in (0, 1]$, we find it wrong for, e.g., $x = 3$ and $n = 2$, when the leftmost part of the inequality is negative, and others are positive. For this reason we show a simple proof here.
\begin{proof}
By~\cite[Lemma~8]{RoweD14} we have $(1 - x)^n \le \frac{1}{1 + xn}$.
Therefore, following the arguments that were used in~\cite[Theorem~9]{RoweD14} we conclude
\begin{align*}
  1 - (1 - x)^n & \ge 1 - \frac{1}{1 + xn} \\
  &= \frac{xn}{1 + xn} = \frac{1}{1 + \frac{1}{xn}}.
\end{align*}
\end{proof}

%
%

\section{Upper Bounds}
\label{sc:ub}

In this section, we prove separately two upper bounds for the runtime of the \ea on the \onemax problem, the first one being valid for all values of $\mu$ and $\lambda$ and the second one giving an improvement for the case that $\lambda$ is large compared to $\mu$, more precisely, that $\lambda/\mu \ge e^e$.

Where not specified differently, we denote the current best fitness in the population by $i$ and the number of best individuals in the population by $j$.

\subsection{Increase of the Number of the Best Individuals}
\label{scn:mu-best}

In this subsection we analyze how the number of individuals on the current-best fitness level increases over time and derive from this two estimates for the time taken for a fitness improvement. We note that often it is much easier to generate an additional individual with current-best fitness by copying an existing one than to generate an individual having strictly better fitness by flipping the right bits. Consequently, in a typical run of the \ea, first the number of best individuals will increase to a certain number and only then it becomes likely that a strict improvement happens.

Since the increase of the number of individuals on the current-best fitness level via producing copies of such best individuals is independent of the fitness function, we formulate our results for the optimization of an arbitrary pseudo-Boolean function and hope that they might find applications in other runtime analyses as well. So let $f : \{0,1\}^n \to \R$ be an arbitrary fitness function which we optimize using the \ea.

Assume that the \ea starts in an arbitrary state where the best individuals have fitness $i$ and there are $j_1$ such individuals in the population. At this point due to the elitism the algorithm cannot decrease the best fitness $i$ and it also cannot decrease the number of the best individuals $j_1$ until it increases the best fitness. Following~\cite[Lemma~2]{Sudholt09} we call an individual \emph{fit} if it has a fitness $i$ or better. For ${j_2 \in \N}$, we define $\tau_{j_1,j_2}(i)$ to be the first time (number of iterations) at which the population of the \ea contains at least $j_2$ fit individuals. We note that this random variable $\tau_{j_1,j_2}(i)$ may depend on the particular initial state of the \ea, but since our results are independent of this initial state (apart from $i$ and $j_1$) we suppress in our notation the initial state.

The time $\tau_{1,\mu}(i)$, that is, the specific case that $j_1 = 1$ and $j_2 = \mu$, is also called the \emph{takeover time} of a new best individual. For this takeover time, Sudholt~\cite[Lemma~2]{Sudholt09} proved the upper bound
\begin{align}\label{eq:tau-dirk}
E[\tau_{1,\mu}(i)] = \lceil\log_5\mu\rceil \left(\frac{32}{1 - \frac{1}{e}} \cdot \frac\mu\lambda + 1\right) = O{\left(\frac{\mu\log\mu}{\lambda} + \log\mu\right)}
\end{align}
for any $i \in [0..n-1]$.

In this section we improve this result by (i)~treating the general case of arbitrary $j_1, j_2 \in [1..\mu]$ and (ii)~by showing an asymptotically smaller bound for the case $\lambda = \omega(\mu)$. In our main analysis of the \ea, we need takeover times for general values of $j_2$ to profit from the event when we get a fitness gain before the population contains only best individuals, which is likely to happen on the lower fitness levels as we show further. The extension to general values of $j_1$ is not needed, but since it does not take extra effort, we do it on the way.

We first prove the following result for arbitrary values of $\mu$ and $\lambda$. We need this result since it allows arbitrary target numbers $j_2$. 

\begin{lemma}\label{lm:mu-best-general}
  Let $i \in [0..n-1]$ and $j_1, j_2 \in [1..\mu]$ with $j_1 < j_2$. Then
  \[
  E[\tau_{j_1, j_2}(i)] \le \frac{2e\mu}{\lambda} \left(\ln\frac{j_2}{j_1} + 1\right) + (j_2 - j_1).
  \]
\end{lemma}

\begin{proof}
To prove this lemma we use Theorem~\ref{th:levels}. For this purpose we define levels $A_{j_1}, \dots, A_{j_2}$. For any $j \in [j_1..j_2 - 1]$ the populations in level $A_j$ have exactly $j$ fit individuals. The level $A_{j_2}$ consists of all populations with at least $j_2$ fit individuals. Note that the \ea cannot go from level $A_j$ to any other level with smaller index, since it cannot decrease the number of the fit individuals due to the elitist selection.

If there are $j$ fit individuals in the population, then the probability $p_1(j)$ to create as one offspring a copy of a fit individual is the probability to select one of $j$ fit individuals as a parent multiplied by the probability not to flip any bit of it during the mutation. Hence,
\begin{align}\label{eq:p1}
  p_1(j) \ge \frac{j}{\mu}\left(1 - \frac 1n \right)^n \ge \frac{j}{2e\mu},
\end{align}
where we used the inequality $(1 - \frac 1n)^n \ge \frac{1}{2e}$ that holds for all $n \ge 2$.

The probability $p_2(j)$ to leave level $A_j$ in one iteration is at least the probability to create a copy of a fit individual as one of the $\lambda$ offspring. Hence, by Lemma~\ref{lm:bernoulli} we have
\begin{align}\label{eq:p2}
  p_2(j) \ge 1 - (1 - p_1(j))^\lambda \ge \frac{1}{1 + \frac{1}{p_1(j)\lambda}} \ge \frac{1}{1 + \frac{2e\mu}{j\lambda}}.
\end{align}

By Theorem~\ref{th:levels} we have
\begin{align*}
  E[\tau_{j_1, j_2}(i)] \le \sum_{j = j_1}^{j_2 - 1} \frac{1}{p_2(j)} \le \sum_{j = j_1}^{j_2 - 1} \left( 1 + \frac{2e\mu}{j\lambda} \right) \le  \frac{2e\mu}{\lambda}\left(\ln\frac{j_2}{j_1} + 1\right) + (j_2 - j_1).
\end{align*}
\end{proof}

We note that in case when $j_1 = 1$ and $j_2 = \mu$ our upper bound is $O(\frac{\mu\log\mu}{\lambda} + \mu)$. This is weaker than the upper bound~\eqref{eq:tau-dirk} given in~\cite[Lemma~2]{Sudholt09} if $\lambda = \omega(\log\mu)$.
Without proof we note that in all other cases the two bounds are asymptotically equal.

The reason that our bound is weaker in some cases is that we do not consider the event that the algorithm generates more than one fit offspring in one iteration, while Sudholt in~\cite[Lemma~2]{Sudholt09} proved that the number of the fit offspring is multiplied by some constant factor in every $32\mu/\lambda$ iterations. The same idea may be used to prove the bound
\begin{align}\label{eq:tau-dirk-modified}
E[\tau_{j_1, j_2}(i)] \le \left\lceil \log_5\frac{j_2}{j_1} \right\rceil \left( \frac{32}{1 - \frac 1e} \cdot \frac\mu\lambda + 1\right) = O{\left(\frac{\mu\log\frac{j_2}{j_1}}{\lambda} + \log\frac{j_2}{j_1} \right)}.
\end{align}
We still prefer to use to Lemma~\ref{lm:mu-best-general} in our proofs, since it gives us a bound that is easier to operate with due to the simpler leading constants of each term, while the greater terms do not affect our main results.

We now give a second bound for the case that $\frac\lambda\mu \ge e^e$. It is asymptotically stronger than~\eqref{eq:tau-dirk} when $\lambda = \omega(\mu)$ and $\mu = \omega(1)$.

\begin{lemma}\label{lm:mu-best-big-lambda}
  Let $\frac{\lambda}{\mu} \ge e^e$. Let $i \in [1..n-1]$ and $j_1, j_2 \in [1..\mu]$ with $j_1 < j_2$. Then
  \[
  E[\tau_{j_1, j_2}(i)] \le 4 \, \frac{\ln\frac{j_2}{j_1}}{\ln \frac{\lambda}{2e\mu}} + 4.
  \]
\end{lemma}

\begin{proof}
Let the current population have $j$ fit individuals. Then by~\eqref{eq:p1} the probability that a fixed offspring is a copy of a fit individual is $p_1(j) \ge \frac{j}{2e\mu}$. Therefore, the number $N$ of fit individuals among the $\lambda$ offspring dominates stochastically a random variable $B$ with binomial distribution $\Bin\left(\lambda,\frac{j}{2e\mu}\right)$.
We have $E[B] = \frac{\lambda j}{2e\mu}$. By Lemma~\ref{lm:multbin}, $\Pr[B \ge  E[B]] \ge \frac{1}{4}$ and thus $\Pr[N \ge \frac{j}{2e\mu}] \ge \frac{1}{4}$.
Consequently, in each iteration with probability at least~$\frac{1}{4}$ the number of the fit individuals in the population is multiplied by a factor of at least $(1 + \frac{\lambda}{2e\mu})$ (but obviously it cannot become greater than $\mu$).

For a formal proof we define the levels $A_1, \dots, A_m$, where
\[
m \coloneqq \bigg\lceil \frac{\ln\frac{j_2}{j_1}}{\ln \left(1 + \frac{\lambda}{2e\mu}\right)} \bigg\rceil + 1.
\]
Level $A_m$ consists of the populations with at least $j_2$ fit individuals. For $k \in [1..m-1]$ the populations of level $A_k$ have exactly $j$ fit individuals, where
\[
j \in \left[j_1\left(1 + \frac{\lambda}{2e\mu}\right)^{k - 1},\ j_1\left(1 + \frac{\lambda}{2e\mu}\right)^k - 1\right],
\]
and $j < j_2$. To leave any level it is enough to multiply the number of the best individuals by $1 + \frac{\lambda}{2e\mu}$, and the probability of this event is at least $\frac 14$. By Theorem~\ref{th:levels} we have
\begin{align*}
E[\tau_{j_1, j_2}(i)] &\le \sum_{k = 1}^{m - 1} 4 = 4\bigg\lceil \frac{\ln\frac{j_2}{j_1}}{\ln \left(1 + \frac{\lambda}{2e\mu}\right)} \bigg\rceil \\
                      &\le 4\,\frac{\ln\frac{j_2}{j_1}}{\ln \frac{\lambda}{2e\mu}} + 4.
\end{align*}

\end{proof}

We note that the proof of Lemma~\ref{lm:mu-best-big-lambda} holds for the weaker assumption $\frac\lambda\mu > 2e$ as well. However in order not to confuse the reader in Section~\ref{scn:fast_upper} where we consider the case $\frac\lambda\mu > e^e$ and where this lemma is used, we formulate Lemma~\ref{lm:mu-best-big-lambda} with unnecessarily stronger condition.

When $j_2 = \mu$ and $j_1 = 1$ the bound yielded by Lemma~\ref{lm:mu-best-big-lambda} is at least as tight as that of~\eqref{eq:tau-dirk}. For the general values of $j_1$ and $j_2$ our bound is at least as tight as the bound~\eqref{eq:tau-dirk-modified}.
When $\lambda / \mu \ge e^e$ the bound~\eqref{eq:tau-dirk-modified} simplifies to $O(\log\frac{j_2}{j_1})$.
If $\lambda = \omega(\mu)$ and $\frac{j_2}{j_1} = \omega(1)$ then we have
\[
4\frac{\ln\frac{j_2}{j_1}}{\ln \frac{\lambda}{2e\mu}} + 4 = o{\left(\log\frac{j_2}{j_1}\right)}.
\]
Therefore, in this case the bound given in Lemma~\ref{lm:mu-best-big-lambda} is asymptotically smaller than~\eqref{eq:tau-dirk-modified}. In all other cases the two bounds are asymptotically equal.

The reason that we have obtained a tighter bound is that we have proven that the number of the fit individuals is multiplied by a more than constant factor with constant probability, while the proof of~\cite[Lemma~2]{Sudholt09} considers only the multiplication by a constant factor.

We now use Lemmas~\ref{lm:mu-best-general} and~\ref{lm:mu-best-big-lambda} to prove estimates for the time it takes to obtain a strictly better individual once the population contains at least one individual of fitness $i$. We define $\tilde T_i$ as the number of iterations before the algorithm finds an individual with fitness greater than $i$, if it already has an individual with fitness $i$ in the population.
As before, this random variable depends on the precise initial state, but since our results do not rely on the initial state, we suppress it in this notation.

To prove upper bounds on $\tilde T_i$,
we estimate the time it takes until some number $\mu_0(i) \in [1..\mu]$ of individuals with fitness at least $i$ are in the population and then estimate the time to find an improving solution from this situation. We phrase our results here in terms of $\mu_0(i)$ and optimize the value of $\mu_0(i)$ in the later subsections.

\begin{corollary}\label{lm:tilde-t-general}
For any $i \in [0..n-1]$ and $\mu_0(i)\in[1..\mu]$, we have
\[
E[\tilde T_i] \le \mu_0(i) +\frac{2e\mu}{\lambda}(\ln(\mu_0(i)) + 1) + \frac{e\mu n}{\lambda (n - i)\mu_0(i)}.
\]
\end{corollary}

\begin{proof}
  Even if the algorithm has only one best individual in the population, in $\tau_{1,\mu_0(i)}(i)$ iterations it will have at least $\mu_0(i)$ individuals with fitness at least $i$.
  Assume that at this time we have no individuals with fitness better than $i$ (since otherwise we are done).
  Let $\tau^+(i)$ be the runtime until the algorithm creates an individual with fitness at least $i + 1$ if it already has at least $\mu_0(i)$ individuals with fitness $i$ in the population.

  In this setting the probability $p'(i)$ that a particular offspring has fitness better than $i$ is at least the probability to choose one of the $\mu_0(i)$ best individuals and to flip only one of $n - i$ zero-bits in it. We estimate
  \[
  p'(i) \ge \frac{\mu_0(i)(n - i)}{\mu n}\left(1-\frac{1}{n}\right)^{n-1}\ge \frac{(n - i)\mu_0(i)}{e\mu n}.
  \]
  By Lemma~\ref{lm:bernoulli} the probability $p''(i)$ to create at least one superior individual among the $\lambda$ offspring is
  \begin{align}\label{eq:p-doble-dash}
  p''(i) &\ge 1 - (1 - p'(i))^\lambda \ge \frac{1}{1+\frac{1}{\lambda p'(i)}} \ge \frac{1}{1 + \frac{e\mu n}{\lambda (n - i)\mu_0(i)}}.
  \end{align}
  With $p''(i)$ we estimate $E[\tau^+(i)] \le \frac{1}{p''(i)}$. Therefore, by Lemma~\ref{lm:mu-best-general} we have
  \begin{align*}
    E[\tilde T_i] &\le E[\tau_{1,\mu_0(i)}(i) + \tau^+(i)] = E[\tau_{1,\mu_0(i)}(i)] + E[\tau^+(i)] \\
                  &\le E[\tau_{1,\mu_0(i)}(i)] + \frac{1}{p''(i)} \\
                  &\le \frac{2e\mu}{\lambda} \left(\ln\mu_0(i) + 1\right) + (\mu_0(i) - 1) + 1 + \frac{e\mu n}{\lambda (n - i)\mu_0(i)} \\
                  &=\mu_0(i) +\frac{2e\mu}{\lambda}(\ln(\mu_0(i)) + 1) + \frac{e\mu n}{\lambda (n - i)\mu_0(i)}.
  \end{align*}
\end{proof}

\begin{corollary}\label{lm:tilde-t-big-lambda}
If $\frac{\lambda}{\mu} > e^e$ then for any $i \in [0..n-1]$ and $\mu_0(i)\in[1..\mu]$, we have
\[
E[\tilde T_i] \le 4\frac{\ln\mu_0(i)}{\ln \frac{\lambda}{2e\mu}} + \frac{e\mu n}{\lambda (n - i)\mu_0(i)} + 5.
\]
\end{corollary}
\begin{proof}
Using the same arguments as in the proof of Corollary~\ref{lm:tilde-t-general} (in particular, the estimate for $p''(i)$ given in~\eqref{eq:p-doble-dash}) and by Lemma~\ref{lm:mu-best-big-lambda} we estimate
\begin{align*}
  E[\tilde T_i] &\le E[\tau_{1,\mu_0(i)}(i) + \tau^+(i)] = E[\tau_{1,\mu_0(i)}(i)] + E[\tau^+(i)] \\
                &\le E[\tau_{1,\mu_0(i)}(i)] + \frac{1}{p''(i)} \\
                &\le 4\frac{\ln\mu_0(i)}{\ln \frac{\lambda}{2e\mu}} + 4 + 1 + \frac{e\mu n}{\lambda (n - i)\mu_0(i)} \\
                &=4\frac{\ln\mu_0(i)}{\ln \frac{\lambda}{2e\mu}} + \frac{e\mu n}{\lambda (n - i)\mu_0(i)} + 5.
\end{align*}
\end{proof}

We note that Lemma~\ref{lm:mu-best-big-lambda} is tight in the sense that we cannot obtain a better upper bound using only the argument of copying the fit individuals.

To formalize this we assume that there is a set $D \subseteq \{0,1\}^n$ of \emph{desired individuals}. We regard the variant EA$^0$ of the \ea which only accepts offspring which are desired individuals identical to their parent. Note that the number of desired individuals in a run of this artificial algorithm can never decrease. Assuming that the initial population of the EA$^0$ contains exactly $j_1$ desired individuals, we define $\tau_{j_1, j_2}^*(i)$ as the number of iterations until the population of the EA$^0$ contains at least $j_2$ desired individuals (unlike before, this notation does not depend on the precise initial population as long as it has exactly $j_1$ desired individuals, but this fact is not important in the following). We show the following result.

\begin{lemma}
Let $\frac\lambda\mu \ge e^e$. Let $j_1, j_2$ be some integer numbers in $[1..\mu]$ such that $j_2 > j_1$. Then
\[
E[\tau_{j_1, j_2}^*(i)] = \Omega\left(\frac{\log\frac{j_2}{j_1}}{\log \frac{\lambda}{\mu}} + 1 \right).
\]
\end{lemma}

\begin{proof}
If $\frac{j_2}{j_1} \le \frac{\lambda}{\mu}$, then
\[
\frac{\log\frac{j_2}{j_1}}{\log \frac{\lambda}{\mu}} + 1 = \Omega(1)
\]
and the claim is trivial, since we need at least one iteration to increase the number of the copies in the population.

Consider $\frac{j_2}{j_1} \ge \frac{\lambda}{\mu} \ge e^e$. Let $j(t)$ be the number of the desired individuals after iteration $t$. We have $j(0) = j_1$. Let $N(t)$ be the number of desired individuals newly created in iteration~$t$. Then $N(t)$ follows a binomial law $\Bin(\lambda, \frac{j(t - 1)}{e_n\mu})$, where $e_n \coloneqq (1 - \frac{1}{n})^{-n} \ge e$.
Hence, we have $E[N(t) \mid j(t - 1)] = \frac{j(t - 1)\lambda}{e_n\mu} \le \frac{j(t - 1)\lambda}{e\mu}$.

For any $t \in \N$ we have $j(t) \le j(t - 1) + N(t)$, where strict inequality occurs only if $j(t - 1) + N(t) > \mu$.
Therefore, we have
\begin{align*}
 E[j(t)] &= E[E[j(t) \mid j(t - 1)]] \le E[E[j(t - 1) + N(t) \mid j(t - 1) ]] \\
         &= E[j(t - 1)] + E[E[N(t) \mid j(t - 1)]] \le E[j(t - 1)] + E\left[\frac{j(t - 1)\lambda}{e\mu}\right] \\
         &= E[j(t - 1)] + \frac{\lambda}{e\mu} \, E[j(t - 1)] = \left(1 + \frac{\lambda}{e\mu}\right)E[j(t - 1)].
\end{align*}
By induction, we obtain
\begin{align*}
 E[j(t)] \le \left(1 + \frac{\lambda}{e\mu}\right)^t j(0) = \left(1 + \frac{\lambda}{e\mu}\right)^t j_1,
\end{align*}
and by Markov's inequality, we have
\begin{align*}
  \Pr[j(t) \ge j_2] \le \frac{E[j(t)]}{j_2} \le \left(1 + \frac{\lambda}{e\mu}\right)^t \frac{j_1}{j_2}.
\end{align*}
For
\[
t \coloneqq \frac{\ln\frac{j_2}{2j_1}}{\ln\left(1 + \frac{\lambda}{e\mu}\right)} = \Omega\left(\frac{\log\frac{j_2}{j_1}}{\log \frac{\lambda}{\mu}}\right)
\]
we obtain
\[
\Pr[j(t) \ge j_2] \le \frac{1}{2}.
\]

Hence, the probability that the EA$^0$ does not obtain $j_2$ desired individuals in $t = \Omega(\frac{\log(j_2/j_1)}{\log (\lambda/\mu)})$ iterations is at least constant. Thus, the expected number of iterations before this happens is at least $\Omega(\frac{\log(j_2/j_1)}{\log (\lambda/\mu)})$.
\end{proof}

\subsection{Unconditional Upper Bound}
\label{scn:unconditional_upper}

Having the results of Section~\ref{scn:mu-best} we first prove the following upper bound, which is valid for all values of $\mu$ and $\lambda$. When $\lambda$ is not significantly larger than $\mu$, then the \ea typically increases the best fitness by at most a constant in each iteration. For this reason, we can use Theorem~\ref{th:levels} and obtain a runtime bound that will turn out to be tight for this case.

\begin{theorem}
\label{thm:general}
The expected number of iterations for the $\left(\mu+\lambda\right)$ EA to optimize the
\onemax problem is
\[
O{\left(\frac{n\log n}{\lambda}+\frac{n\mu}{\lambda}+n\right)}.
\]
\end{theorem}
\begin{proof}

To use Theorem~\ref{th:levels} we define levels $A_0, \dots, A_n$ such that level $A_i$, $i \in [0..n]$, consists of all populations having maximum fitness equal to~$i$. In Corollary~\ref{lm:tilde-t-general} we have already estimated the expected times $E[\tilde T_i]$ the \ea takes to leave these levels. These estimates depended on the number $\mu_0(i)$ of individuals of fitness $i$ we aim at before leaving the level.
By choosing suitable values for the $\mu_0(i)$ we prove our bound.

The choice of $\mu_0(i)$ is guided by the following trade-off. If we choose $\mu_0(i) = \mu$, then after having $\mu$ best individuals in the population we have the highest probability to find a better individual. However, we pay for it with the time we spend on obtaining $\mu$ copies of the best individual. On the other hand, if we choose $\mu_0(i) = 1$ we do not spend any iteration filling the population with copies of the best individual, but we have a low chance to increase the current fitness. How this trade-off is optimally resolved, and hence the optimal value of $\mu_0(i)$, depends on the probability to create a better individual and thus on the current fitness~$i$.

We distinguish three cases depending on current fitness $i$. The ``milestones'' which mark the transition between these cases are the fitness values $i = \lceil n - \frac{n}{2 + \lambda/(e\mu)} \rceil$ and $i = \lfloor n - \frac{n}{\mu(2 + \lambda/e)} \rfloor$. While the best fitness is below the first milestone, the probability to increase the fitness is so high that we do not need to have more than one best individual in the population. Beyond the second milestone this probability is so low that we better spend the time to fill the population with the copies of the best individual. Between the two milestones we have to find a suitable value of $\mu_0(i)$ to give a balanced trade-off.

To simplify the notation, we define $\Delta_i\coloneqq1+\sqrt{1+\frac{n\lambda}{e(n - i)\mu}}$. Note that
\begin{align}\label{eq:delta}
1 + \sqrt{\frac{n}{n - i}}\sqrt{\frac{\lambda}{e\mu}} \le \Delta_i \le 1 + \sqrt{\frac{n}{n - i}}\sqrt{1+\frac{\lambda}{e\mu}}.
\end{align}
This value of $\Delta_i$ arises from the computation of the derivative of the upper bound on $E[\tilde T_i]$ from Corollary~\ref{lm:tilde-t-general}, which is needed to find the optimal value of $\mu_0(i)$ in the second case, when the current fitness is between the two milestones.

\textbf{For $i \le \lceil n - \frac{n}{2 + \lambda/(e\mu)} \rceil$ we define $\mu_0(i) \coloneqq 1$.}
By Corollary~\ref{lm:tilde-t-general} we have
\[
E[\tilde T_i] \le \frac{e\mu n}{\lambda (n - i)} + \frac{2e\mu}{\lambda} + 1.
\]
Let $T_1$ be the number of iterations before the \ea finds an individual with fitness greater than $\lceil n - \frac{n}{2 + \lambda/(e\mu)} \rceil$ for the first time. Then by Theorem~\ref{th:levels} we have
\begin{align*}
 E[T_1] &\le \sum_{i = 0}^{\lceil n - \frac{n}{2 + \lambda / e\mu}\rceil} E[\tilde T_i]
  			 \le \sum_{i = 0}^{\lceil n - \frac{n}{2 + \lambda / e\mu}\rceil} \left(\frac{e\mu n}{\lambda (n - i)} + \frac{2e\mu}{\lambda} + 1\right) \\
  			&\le \frac{e\mu n}{\lambda} \left(\ln(n) - \ln\left(\frac{n}{2 + \lambda / e\mu}\right) + 1\right) + \frac{2e\mu}{\lambda}n + n \\
				&= \frac{e\mu n}{\lambda} \ln\left(2 + \frac{\lambda}{e\mu}\right) + \frac{3e\mu n}{\lambda} + n  \\
				&= O{\left(\frac{\mu n}{\lambda}\right)} + O(n),
\end{align*}
where we used the estimate $\frac{e\mu}{\lambda}\ln(2 + \frac{\lambda }{ e\mu}) = O(1 + \frac{\mu }{ \lambda})$ that holds for any asymptotic behavior of $\mu/\lambda$.

\textbf{For $\lceil n - \frac{n}{2 + \lambda/(e\mu)} \rceil < i \le \lfloor n - \frac{n}{\mu(2 + \lambda/e)} \rfloor$ we define $\mu_0(i)\coloneqq\lceil\frac{n}{(n - i)\Delta_i}\rceil$.} By Corollary~\ref{lm:tilde-t-general} we have
\begin{align*}
E[\tilde T_i] \le \frac{n}{(n - i)\Delta_i} + 1 + \frac{2e\mu}{\lambda}\left(\ln\frac{n}{(n - i)\Delta_i} + 2\right) + \frac{e\mu\Delta_i}{\lambda}.
\end{align*}

By~\eqref{eq:delta}, we have
\begin{align}\label{eq:T_i-middle}
  \begin{split}
    E[\tilde T_i] &\le  \frac{n}{(n - i) \left(1 + \sqrt{\frac{n}{n - i}}\sqrt{\frac{\lambda}{e\mu}}\right)} + 1 \\
                  &+ \frac{2e\mu}{\lambda}\left(\ln\frac{n}{(n - i)\left(1 + \sqrt{\frac{n}{n - i}}\sqrt{\frac{\lambda}{e\mu}}\right)} + 2\right) \\
                  &+ \frac{e\mu}{\lambda}\left(1 + \sqrt{\frac{n}{n - i}}\sqrt{1+\frac{\lambda}{e\mu}}\right).
  \end{split}
\end{align}

Let $T_2$ be the number of iterations until the \ea finds an individual with fitness greater than $\lfloor n - \frac{n}{\mu(2 + \lambda/e)} \rfloor$ for the first time if it already has an individual with fitness greater than $\lceil n - \frac{n}{2 + \lambda/(e\mu)} \rceil$ in the population.
By Theorem~\ref{th:levels} and by~\eqref{eq:T_i-middle}, we obtain
\begin{align}\label{eq:t2}
\begin{split}
 E[T_2] &\le  \sum_{i = \lceil n - \frac{n}{2 + \lambda/(e\mu)}\rceil + 1}^{\lfloor n - \frac{n}{\mu(2 + \lambda/e)} \rfloor} E[\tilde T_i] \\
        &\le n \sum_{i = \lceil n - \frac{n}{2 + \lambda/(e\mu)}\rceil + 1}^{\lfloor n - \frac{n}{\mu(2 + \lambda/e)} \rfloor} \frac{1}{(n - i) \left(1 + \sqrt{\frac{n}{n - i}}\sqrt{\frac{\lambda}{e\mu}}\right)} \\
        &+ \frac{e\mu}{\lambda} \sum_{i = \lceil n - \frac{n}{2 + \lambda/(e\mu)}\rceil + 1}^{\lfloor n - \frac{n}{\mu(2 + \lambda/e)} \rfloor} \left(1 + \sqrt{\frac{n}{n - i}}\sqrt{1+\frac{\lambda}{e\mu}}\right) \\
        &+ \frac{2e\mu}{\lambda} \sum_{i = \lceil n - \frac{n}{2 + \lambda/(e\mu)}\rceil + 1}^{\lfloor n - \frac{n}{\mu(2 + \lambda/e)} \rfloor} \ln\left(\frac{n}{(n - i)\Delta_i}\right) + \frac{2e\mu}{\lambda}2n + n. \\
\end{split}
\end{align}

We regard three sums in~\eqref{eq:t2} separately. First, by the estimate $\sum_{i = 1}^n 1/\sqrt{i} \le 1 + \int_1^n (1/\sqrt{x}) dx < 2\sqrt{n}$, we obtain
\begin{align}\label{eq:t2_1}
\begin{split}
 \sum_{i = \lceil n - \frac{n}{2 + \lambda/(e\mu)}\rceil + 1}^{\lfloor n - \frac{n}{\mu(2 + \lambda/e)} \rfloor} &\frac{1}{(n - i) \left(1 + \sqrt{\frac{n}{n - i}}\sqrt{\frac{\lambda}{e\mu}}\right)}  \\
 &\le \sum_{i = \lceil n - \frac{n}{2 + \lambda/(e\mu)}\rceil + 1}^{\lfloor n - \frac{n}{\mu(2 + \lambda/e)} \rfloor} \frac{1}{(n - i) \sqrt{\frac{n}{n - i}}\sqrt{\frac{\lambda}{e\mu}}} \\
 &= \sqrt{\frac{e\mu}{\lambda n}}  \sum_{i = \lceil n - \frac{n}{2 + \lambda/(e\mu)}\rceil + 1}^{\lfloor n - \frac{n}{\mu(2 + \lambda/e)} \rfloor} \frac{1}{\sqrt{n - i}} \\
 &\le \sqrt{\frac{e\mu}{\lambda n}} \cdot 2\sqrt{\frac{n}{2 + \lambda/(e\mu)}} \le 2 \sqrt{\frac{e\mu}{\lambda n}}\sqrt{\frac{n}{\lambda/(e\mu)}} = 2e\frac{\mu}{\lambda}.
\end{split}
\end{align}

To analyze the second sum we also use the estimate $\frac{1 + t}{2 + t} < 1$ valid for all $t \in [0, +\infty)$. We obtain

\begin{align}\label{eq:t2_2}
	\begin{split}
		\sum_{i = \lceil n - \frac{n}{2 + \lambda/(e\mu)}\rceil + 1}^{\lfloor n - \frac{n}{\mu(2 + \lambda/e)} \rfloor} & \left(1 + \sqrt{\frac{n}{n - i}}\sqrt{1+\frac{\lambda}{e\mu}}\right)\\
		&\le n + \sqrt{n\left(1 + \frac{\lambda}{e\mu}\right)} \sum_{i = \lceil n - \frac{n}{2 + \lambda/(e\mu)}\rceil + 1}^{\lfloor n - \frac{n}{\mu(2 + \lambda/e)} \rfloor} \frac{1}{\sqrt{n - i}} \\
		&\le n + \sqrt{n\left(1 + \frac{\lambda}{e\mu}\right)} \cdot 2\sqrt{\frac{n}{2 + \lambda/(e\mu)}} \\
		&= n + 2n \sqrt{\frac{1 + \lambda/(e\mu)}{2 + \lambda/(e\mu)}} \le 3n.
	\end{split}
\end{align}

For the last sum we use the logarithmic version of Stirling's formula, that is, $\ln(n!) = n\ln(n) - n + O(\log(n))$ (see, e.g.~\cite{Robbins55} or~\cite[Theorem 1.4.10]{Doerr20bookchapter}), and the estimate $\frac{\ln(2 + t) + 2}{2 + t} \le 2$ for all $t \in [0, +\infty)$. We obtain

\begin{align}\label{eq:t2_3}
\begin{split}
	\sum_{i = \lceil n - \frac{n}{2 + \lambda/(e\mu)}\rceil + 1}^{\lfloor n - \frac{n}{\mu(2 + \lambda/e)} \rfloor} & \ln\left(\frac{n}{(n - i)\Delta_i}\right) \\
	&\le \sum_{i = \lceil n - \frac{n}{2 + \lambda/(e\mu)}\rceil + 1}^{\lfloor n - \frac{n}{\mu(2 + \lambda/e)} \rfloor} \ln\left(\frac{n}{(n - i)}\right) \\
	&\le \left\lceil \frac{n}{2 + \lambda/(e\mu)} \right\rceil \ln(n) - \ln\left(\prod_{i = \lceil n - \frac{n}{2 + \lambda/(e\mu)}\rceil}^n (n - i) \right)\\
	&\le \left\lceil \frac{n}{2 + \lambda/(e\mu)} \right\rceil \ln(n) - \ln\left(\left\lceil \frac{n}{2 + \lambda/(e\mu)} \right\rceil! \right) \\
	&= \left\lceil \frac{n}{2 + \lambda/(e\mu)} \right\rceil\left(\ln(n) - \ln\left\lceil \frac{n}{2 + \lambda/(e\mu)} \right\rceil + 1 \right) \\
	&+ O{\left(\log\left\lceil \frac{n}{2 + \lambda/(e\mu)} \right\rceil  \right)} \\
	&\le \frac{n}{2 + \lambda/(e\mu)} \left(\ln\left(2 + \frac{\lambda}{e\mu}\right) + 2\right) + o(n) \\
	&\le 2n + o(n).
\end{split}
\end{align}

Finally, by putting~\eqref{eq:t2_1},~\eqref{eq:t2_2} and~\eqref{eq:t2_3} into~\eqref{eq:t2} we obtain
\begin{align*}
E[T_2] &\le n \cdot 2e\frac{\mu}{\lambda} + \frac{e\mu}{\lambda} \cdot 3n + \frac{2e\mu}{\lambda} (2n + o(n)) + \frac{4e\mu n}{\lambda} + n \\
			 &= \frac{13e\mu n}{\lambda} + n + o{\left(\frac{\mu n}{\lambda}\right)} = O{\left(\frac{\mu n}{\lambda} + n\right)}.
\end{align*}

\textbf{For $n - 1 \ge i > \lfloor n - \frac{n}{\mu(2 + \lambda/e)} \rfloor$ we define $\mu_0(i) \coloneqq \mu.$}
Note that this case can only appear when $\frac{n}{\mu(2 + \lambda/e)} \ge 1$ and thus $\mu \le \frac{n}{(2 + \lambda/e)} = O(n/\lambda)$.
By Corollary~\ref{lm:tilde-t-general} the expected waiting time for a fitness gain is at most
\[
E[\tilde T_i] \le \mu + \frac{2e\mu}{\lambda}(\ln(\mu) + 1) + \frac{en}{\lambda (n - i)}.
\]
Let $T_3$ be the number of iterations until the \ea finds the optimum starting from the moment when it has an individual with fitness greater than $\lfloor n - \frac{n}{\mu(2 + \lambda/e)} \rfloor$ in the population.
Then by Theorem~\ref{th:levels} we have
\begin{align*}
  \begin{split}
 E[T_3] &\le \sum_{i = \lfloor n - \frac{n}{\mu(2 + \lambda / e)}\rfloor + 1}^{n - 1} E[\tilde T_i]\\
 				&\le \sum_{i = \lfloor n - \frac{n}{\mu(2 + \lambda / e)}\rfloor + 1}^{n - 1} \left(\mu + \frac{2e\mu}{\lambda}(\ln(\mu) + 1) + \frac{en}{\lambda(n - i)}\right) \\
 				&\le \mu \frac{n}{\mu(2 + \lambda / e)} + \frac{2e\mu}{\lambda}(\ln(\mu) + 1) \frac{n}{\mu(2 + \lambda / e)} \\
				& \quad + \frac{en}{\lambda}\left(\ln\frac{n}{\mu(2 + \lambda / e)} + 1\right) \\
				& = O{\left(\frac{n}{\lambda}\right)} + O{\left(\frac{n \log \mu}{\lambda^2}\right)} + O{\left(\frac{n \log n}{\lambda}\right)} \\
				& = O{\left(\frac{n \log \mu}{\lambda^2}\right)} + O{\left(\frac{n \log n}{\lambda}\right)} = O{\left(\frac{\mu n}{\lambda}\right)} + O{\left(\frac{n \log n}{\lambda}\right)}.
  \end{split}
\end{align*}

Summing the expected runtimes for all cases, we obtain the upper bound for the expected total runtime.

\begin{align*}
E[T] &\le E[T_1] + E[T_2] + E[T_3] \\
&= O{\left(\frac{\mu n}{\lambda} + n\right)}  + O{\left(\frac{\mu n}{\lambda} + n\right)}+ O{\left(\frac{\mu n}{\lambda} + \frac{n\log n}{\lambda}\right)}\\
&= O{\left(\frac{n\log n}{\lambda}+\frac{\mu n}{\lambda}+n\right)}.
\end{align*}
\end{proof}

\subsection{Upper Bound with Large $\lambda$}
\label{scn:fast_upper}

In this section we consider the case when $\frac{\lambda}{\mu} > e^e$. Due to the large number of offspring the algorithm performs significantly better in this case. The first reason of this speed-up is that the algorithm can now gain several fitness levels in one iteration with high probability when the current-best fitness is small. The second reason is the faster increase of the number of best individuals, see Corollary~\ref{lm:tilde-t-big-lambda}. 

These two observations allow us to prove the following upper bound on the runtime.

\begin{theorem}
\label{thm:fast}
If $\frac{\lambda}{\mu}\geq e^e$ then the expected number of iterations for the
$\left(\mu+\lambda\right)$ EA to optimize the \onemax problem is
 \begin{equation*}
O{\left(\frac{n\log\log \frac{\lambda}{\mu}}{\log\frac{\lambda}{\mu}} +
 \frac{n\log n}{\lambda}
 \right)}.
\end{equation*}
\end{theorem}

Note that the bound given in Theorem~\ref{thm:fast} is asymptotically the same as the bound given in Theorem~\ref{thm:general} when $\frac\lambda\mu = \Theta(1)$. The difference between the two bounds becomes asymptotically significant only when $\frac\lambda\mu = \omega(1)$. Therefore it does not matter which constant we choose to distinguish the fast regime of the algorithm. The main purpose of the choice of $e^e$ as a border value is to simplify the proofs and to improve their readability. However without proof we note that all arguments used in this section hold also for the smaller values of $\frac\lambda\mu$ which are greater than $2e$.

To prove Theorem~\ref{thm:fast} we split the optimization process into four phases. Each phase corresponds to some range of the best fitness values, and the phase transition occurs at fitness values $n - \frac{n}{\ln\frac\lambda\mu}$, $n -\frac{\mu n}\lambda$ and $n - \frac n\lambda$. In each phase the \ea has a specific behavior, so we analyze each phase separately in the following four lemmas. 

During the first phase, while the fitness of the best individual is below $n~-~\frac{n}{\ln \frac{\lambda}{\mu}}$, regardless of the number of best individuals, with constant probability we generate an offspring increasing the best fitness in the population by at least $\gamma \coloneqq \lfloor\frac{\ln\frac{\lambda}{\mu}}{2\ln\ln \frac{\lambda}{\mu}}\rfloor$. So we need not more than an expected number of $O(\frac n\gamma)$ iterations to finish the first phase.

Let $R_1$ be the runtime of the \ea until it finds an individual with fitness at least $n - \frac{n}{\ln \frac{\lambda}{\mu}}$, in other words, the duration of the first phase. We prove the following upper bound on the expected value of $R_1$.

\begin{lemma}[Phase 1]\label{lm:phase1}
	If $\frac{\lambda}{\mu}\geq e^e$, then we have
	\[
		E[R_1] = O{\left(\frac{n\log\log \frac{\lambda}{\mu}}{\log \frac{\lambda}{\mu}}\right)}.
	\]
\end{lemma}
\begin{proof}
To use Theorem~\ref{th:levels}, we split the space of populations $S$ into levels $A_1, \dots A_m$, where
\[
m \coloneqq \bigg\lceil \frac{\lfloor n - \frac{n}{\ln\frac\lambda\mu}\rfloor}{\gamma} \bigg\rceil + 1.
\]
If $k < m$, then the populations of level $A_k$ have the fitness of the best individual in $[(k - 1)\gamma..k\gamma - 1]$ (but less than $n - \frac{n}{\ln \frac{\lambda}{\mu}}$). The level $A_m$ consists of all populations containing an individual of fitness at least $n - \frac{n}{\ln \frac{\lambda}{\mu}}$.

To show that we have a constant probability to leave any level, we consider the probability that a particular offspring has a fitness exceeding the current best fitness $i$ by at least $\gamma$. This is at least the probability to choose one of the best individuals and to flip exactly $\gamma$ zero-bits in it and not to flip the other $n - \gamma$ bits, namely
\[
\binom{n - i}{\gamma}\frac{j}{\mu n^\gamma}\left(1-\frac{1}{n}\right)^{n-\gamma}
\geq \frac{j}{e\mu}\left(\frac{n - i}{n\gamma}\right)^\gamma \eqqcolon p_\gamma(i).
\]

The probability to increase the best fitness by at least $\gamma$ with one of $\lambda$ offspring is at least $1 - (1 - p_\gamma(i))^\lambda$. Thus, by Lemma~\ref{lm:bernoulli}, the expected number of iterations for this to happen is not larger than
\[
\frac{1}{1 - (1 - p_\gamma(i))^\lambda} \le 1+e\frac{\mu}{\lambda}\left(\frac{n\gamma}{n - i}\right)^\gamma.
\]

Since  $\frac{\lambda}{\mu} \ge e^e$, we have $\gamma = \lfloor\frac{\ln
\frac{\lambda}{\mu}}{2\ln\ln \frac{\lambda}{\mu}}\rfloor \geq \lfloor \frac e2 \rfloor = 1$. Using this and the estimate $\frac{n}{n - i} \le \ln\frac\lambda\mu$ valid during this phase, we compute
\begin{align*}
\left(\frac{n\gamma}{n - i}\right)^\gamma &\leq \exp\left(\gamma\ln\left(\gamma\ln\frac{\lambda}{\mu}\right)\right) \le \exp \left( \frac{\ln \frac{\lambda}{\mu}}{2\ln\ln \frac{\lambda}{\mu}}\ln\left(\frac{\ln^2 \frac{\lambda}{\mu}}{2\ln\ln \frac{\lambda}{\mu}}\right)\right) \\
&\le \exp \left( \frac{\ln \frac{\lambda}{\mu}}{2\ln\ln \frac{\lambda}{\mu}} 2\ln\ln \frac{\lambda}{\mu}\right) = \exp\left(\ln\frac{\lambda}{\mu}\right) = \frac{\lambda}{\mu}.
\end{align*}
Therefore, the expected
time to increase the fitness by $\gamma$ (and thus to leave level $A_k$ for any $k < m$) is at most $1 + e$. Summing over the levels $A_1, \dots, A_{m - 1}$ , by Theorem~\ref{th:levels} we have
\[
 E[R_1] \le \sum_{k = 1}^{m - 1} (1 + e) < (1 + e)m < (1+e)\frac{n}{\gamma} = O{\left(\frac{n\log\log \frac{\lambda}{\mu}}{\log \frac{\lambda}{\mu}}\right)}.
\]

\end{proof}

Having found an individual with fitness at least $n - \frac{n}{\ln \frac{\lambda}{\mu}}$, we enter the second phase. Due to the elitist selection, the minimum fitness in the population does not decrease, so there is no risk of a fall-back into the first phase.

In the second phase, due to the smaller distance from the optimum, fitness gains by more than a constant are too rare to be exploited profitably. However, even when we only have one best individual in the population, the probability to create at least one better individual in one iteration will still be constant. Consequently, we do not need the arguments of Section~\ref{scn:mu-best} analyzing how the the number of best individuals grows. This phase ends when the best fitness in the population is $n - \frac{\mu n}{\lambda}$ or more.

Let $R_2$ be the runtime of the \ea until it finds an individual with fitness at least $n - \frac{\mu n}{\lambda}$ starting from the moment when it has an individual with fitness at least $n - \frac{n}{\ln \frac{\lambda}{\mu}}$ in the population. In other words, $R_2$ is the duration of the second phase.

\begin{lemma}[Phase 2]\label{lm:phase2}
	If $\frac{\lambda}{\mu}\geq e^e$, then we have
	\[
		E[R_2] = O{\left(\frac{n}{\log\frac{\lambda}{\mu}}\right)}.
	\]
\end{lemma}
\begin{proof}
For
\[
i \in \left[ \bigg\lceil n - \frac{n}{\ln\frac\lambda\mu}\bigg\rceil .. \bigg\lceil n - \frac{\mu n}{\lambda} \bigg\rceil - 1 \right],
\]
the level $B_i$ is defined as the set of all populations in which the best individuals have fitness $i$. For $i = \lceil n - \frac{\mu n}{\lambda} \rceil$ let the level $B_i$ consist of all populations with best fitness at least $i$.

By Corollary~\ref{lm:tilde-t-big-lambda} and defining $\mu_0(i) \coloneqq 1$ for all $i$, we have
\[
E[\tilde T_i] \le  \frac{e\mu n}{\lambda (n - i)} + 5 \le e + 5,
\]
where the last estimate follows from $i \le n - \frac{\mu n}{\lambda}$.
Therefore, by Theorem~\ref{th:levels}
\begin{align*}
 E[R_2] &\le \sum\limits_{i = \lceil n - n/\ln \frac{\lambda}{\mu} \rceil}^{\lceil n - n\mu/\lambda\rceil - 1} E[\tilde T_i] \le (5 + e)\frac{n}{\ln\frac{\lambda}{\mu}} = O{\left(\frac{n}{\log\frac{\lambda}{\mu}}\right)}.
\end{align*}

\end{proof}

After completion of the second phase, generating a strictly better individual is so difficult that it pays off (in the analysis) to wait for more than one best individual in the population. More precisely, depending on the current best fitness $i$ we define a number $\mu_0(i)$ and compute the time to reach  $\mu_0(i)$ best individuals and argue that the expected time to generate a strict improvement when at least $\mu_0(i)$ best individuals are in the population is only constant. Since, as discussed in Section~\ref{scn:mu-best}, specifically in Lemma~\ref{lm:mu-best-big-lambda}, the number of the best individuals in the population roughly increases by a factor $(1 + \frac{\lambda}{2e\mu})$ in each iteration, the algorithm obtains $\mu_0(i)$ individuals reasonably fast.

Let $R_3$ be the runtime of the \ea until it finds an individual with fitness at least $n - \frac{n}{\lambda}$, the end of the third phase, starting from the  moment when it has an individual with fitness at least $n - \frac{\mu n}{\lambda}$ in the population.
\begin{lemma}[Phase 3]\label{lm:phase3}
	If $\frac{\lambda}{\mu}\geq e^e$, then we have
	\[
		E[R_3] = O{\left(\frac{\mu n}{\lambda}\right)}.
	\]
\end{lemma}
\begin{proof}

During this phase the best fitness $i$ in the population satisfies
\[
n - \frac{\mu n}{\lambda} \le i < n - \frac n\lambda,
\]
which implies
\begin{align}\label{eq:ph3}
\frac{\lambda}{\mu} \le \frac{n}{n - i} < \lambda.
\end{align}
For these values of $i$ we define $\mu_0(i) \coloneqq \lceil \frac{n\mu}{(n - i)\lambda}\rceil$. Note that $\mu_0(i) \in [1..\mu]$.

For
\[
i \in \left[ \lceil n - \frac{\mu n}{\lambda} \rceil .. \lceil n - \frac{n}{\lambda} \rceil - 1\right], \\
\]
level $C_i$ is defined as a set of all populations in which the best individuals have fitness $i$. For $i = \lceil n - \frac{n}{\lambda} \rceil$ let the level $C_i$ consist of all populations with best fitness at least $i$.

By Corollary~\ref{lm:tilde-t-big-lambda} and by the definition of $\mu_0(i)$ we have
\begin{align*}
E[\tilde T_i] &\le 4\frac{\ln\mu_0(i)}{\ln \frac{\lambda}{2e\mu}} + \frac{e\mu n}{\lambda (n - i)\mu_0(i)} + 5 \\
              &\le \frac{4}{\ln\frac\lambda{2e\mu}} \left(\ln\frac{n \mu}{(n - i) \lambda} + 1 \right) + e + 5.
\end{align*}

By Theorem~\ref{th:levels} we obtain
\begin{align*}
	E[R_3] &\le \sum_{i = \lceil n - n\mu/\lambda \rceil}^{\lceil n - n/\lambda\rceil - 1} \tilde T_i \\
         &\le \sum_{i = \lceil n - n\mu/\lambda \rceil}^{\lceil n - n/\lambda\rceil - 1} \left( \frac{4}{\ln\frac\lambda{2e\mu}} \left(\ln\frac{n \mu}{(n - i) \lambda} + 1 \right) + e + 5 \right) \\
         &\le \frac{4}{\ln\frac\lambda{2e\mu}} \left(\frac{n \mu}{\lambda} + \sum_{i = \lceil n - n\mu/\lambda \rceil}^{\lceil n - n/\lambda\rceil - 1} \ln\frac{n \mu}{(n - i)  \lambda}\right) + \frac{n \mu}{\lambda} (e + 5).
\end{align*}

We estimate $\sum\limits_{i = \lceil n - n\mu/\lambda \rceil}^{\lceil n - n/\lambda\rceil - 1} \ln \frac{n\mu}{(n - i)\lambda}$ using Stirling's formula as in~\eqref{eq:t2_3}.
We also notice that this phase occurs only when $\frac{n\mu}{\lambda} > 1$, thus we have $(\ln\frac{n\mu}{\lambda} - \ln\lfloor\frac{n\mu}{\lambda}\rfloor) \le 1$. Hence, we obtain.

\begin{align*}
\sum\limits_{i = \lceil n - n\mu/\lambda \rceil}^{\lceil n - n/\lambda\rceil - 1} \ln \frac{n\mu}{(n - i)\lambda} & \le \sum\limits_{i = 1}^{\lfloor n\mu/\lambda \rfloor}\ln \frac{n\mu}{i\lambda} \\
&= \left\lfloor\frac{n\mu}{\lambda}\right\rfloor \ln \frac{n\mu}{\lambda} - \left\lfloor\frac{n\mu}{\lambda}\right\rfloor \ln \left\lfloor\frac{n\mu}{\lambda}\right\rfloor \\
&+ \left\lfloor\frac{n\mu}{\lambda}\right\rfloor + O{\left(\log\left\lfloor\frac{n\mu}{\lambda}\right\rfloor\right)} \\
&= \left\lfloor\frac{n\mu}{\lambda}\right\rfloor \left(\ln\frac{n\mu}{\lambda} - \ln\left\lfloor\frac{n\mu}{\lambda}\right\rfloor + 1 \right) + o{\left(\frac{\mu n}{\lambda}\right)} \\
&\le 2\frac{n\mu}{\lambda} + o{\left(\frac{\mu n}{\lambda}\right)}.
\end{align*}

Therefore,
\begin{align*}
E[R_3] &\le (5 + e)\frac{\mu n}{\lambda}  + 4\frac{2\frac{n\mu}{\lambda} + o{\left(\frac{\mu n}{\lambda}\right)} + \frac{n\mu}{\lambda}}{\ln \frac{\lambda}{2e\mu}} = O{\left(\frac{\mu n}{\lambda}\right)}.
\end{align*}
\end{proof}

When the algorithm is closer to the optimum than in the third phase, then we cannot expect to have a constant probability for a strict fitness improvement even when the whole population consists of individuals of best fitness. In this forth and last phase, we thus always wait (in the analysis) until the population only contains best individuals and then estimate the expected time for an improvement. We denote by $R_4$ the runtime until the algorithm finds the optimum if it already has an individual with fitness at least $n - \frac{n}{\lambda}$ in the population.

\begin{lemma}[Phase 4]\label{lm:phase4}
	If $\frac{\lambda}{\mu}\geq e^e$ then
	\[
		E[R_4] = O{\left(\frac{n \log n }{\lambda}\right)}.
	\]
\end{lemma}
\begin{proof}
For
\[
	i \in \left[ \bigg\lceil n - \frac{n}{\lambda} \bigg\rceil..n - 1\right] \\
\]
we define level $D_i$ as a set of all populations in which the best individuals have fitness $i$. We also define $\mu_0(i) = \mu$ for these values of $i$.

By Corollary~\ref{lm:tilde-t-big-lambda} we have
\[
E[\tilde T_i] \le 4\frac{\ln\mu}{\ln \frac{\lambda}{2e\mu}} + \frac{en}{\lambda (n - i)} + 5.
\]

Therefore, by Theorem~\ref{th:levels}, we obtain

\begin{align*}
E[R_4] &\le \sum_{i= \lceil n - \frac{n}{\lambda} \rceil}^{n - 1} \left(\frac{4\ln \mu}{\ln \frac{\lambda}{2e\mu}} + \frac{e n}{\lambda (n - i)} + 5\right) \\
&\le \frac{4n\ln\mu}{\lambda\ln \frac{\lambda}{2e\mu}} + \frac{en(\ln \frac{n}{\lambda} + 1)}{\lambda} + \frac{5n}{\lambda} \\
&= O{\left(\frac{n}{\log\frac\lambda\mu}\right)} + O{\left(\frac{n \log n }{\lambda}\right)}.
\end{align*}
\end{proof}

Finally, we prove Theorem~\ref{thm:fast}.
\begin{proof}[Proof (Theorem~\ref{thm:fast})]

Since we consider an elitist algorithm that cannot reduce the best fitness, by linearity of expectation and Lemmas~\ref{lm:phase1} to~\ref{lm:phase4} we have
\[
E[T] \le E[R_1] + E[R_2] + E[R_3] + E[R_4] = O{\left(\frac{n\log\log \frac{\lambda}{\mu}}{\log \frac{\lambda}{\mu}}
+ \frac{n\log n}{\lambda}
 \right)}.
 \]
 \end{proof}

\subsection{Comparison With Other Upper Bounds}\label{sec:comparison-upper}

We first note that our upper bound
\[O\bigg(\frac{n\log n}{\lambda}+\frac{n}{\lambda / \mu} +
\frac{n\log^+\log^+ (\lambda/ \mu)}{\log^+ (\lambda / \mu)}\bigg)\]
for the runtime of the \ea on \onemax subsumes the known bounds
\[O(n\log n + \mu n)\]
for the \mea~\cite{Witt06} and 
\[O{\left(\frac{n \log n}{\lambda} + \frac{n \log^+\log^+\lambda}{\log^+\lambda}\right)}\]
for the \lea~\cite{DoerrK15}.

We are not aware of any previous result for the \ea for general values of $\mu$ and $\lambda$.
Using a standard domination argument, however, the results of Dang and Lehre~\cite{DangL16} for the \commaea can be extended to the \ea. Recall that the \commaea differs from the \ea only in the selection mechanism, which disallows the \commaea to select any parent individual into the next population, even the ones which are better than all $\lambda$ offspring. This imposes a constraint on the parameters requiring $\lambda$ to be at least $\mu$. For the case that $\lambda > (1+\eps) e\mu$, $\eps > 0$ a constant\footnote{In~\cite{DangL16} only $\lambda > e\mu$ is required, but from analyzing their proof we suspect that the stronger condition is implicitly used.},  and $\lambda = \Omega(\log n)$, Dang and Lehre~\cite[Theorem~14]{DangL16} proved that the \commaea within an expected number of
\[O(n \log\lambda)\]
iterations finds the optimum of \onemax.

Since the \ea uses elitist selection, it can be shown that the fitness values of its population always stochastically dominate those of the population of the \commaea. More precisely, for a run of the \ea let us for $i \in [1..\mu]$ and $t \in \N$ denote by $f_{it}$ the fitness of the $i$-th individual in the parent population after iteration $t$, where we assume that the individuals are sorted by decreasing fitness. Let us denote by $f'_{it}$ the same for the \commaea. Then for all $i$ and $t$, the random variable $f_{it}$ stochastically dominates $f'_{it}$. This can be shown via coupling in a similar fashion as in the proof of Theorem~23 in~\cite{Doerr19tcs}.
Thus the upper bound given by Dang and Lehre is also valid for the \ea. For the case $\lambda > (1+\eps)e\mu$ and $\lambda = \Omega(\log n)$ regarded by Dang and Lehre, our bound becomes
\[
O(n),
\]
which is of an asymptotically slightly smaller order than that of~\cite{DangL16}.

\section{Lower Bounds}
\label{scn:lower}

In this section, we show the lower bounds corresponding to the upper bounds we proved in the previous section. They in particular imply the lower bounds for the \mea given in~\cite{Witt06} and the \lea given in~\cite{DoerrK15}. Hence our proof method is a unified approach to both these algorithms as well. The arguments we use do not consider selection phase at all, thus they hold also for all functions with a unique optimum and for other selection mechanisms, including the \commaea.

The main problem when proving lower bounds for population-based algorithms is that many individuals which are created during the run of the EA are removed at some stage by selection operations. This creates a complicated population dynamics, which is very hard to follow via mathematical means.

One way to overcome this difficulty is to try to disregard the effect of selection and instead regard an optimistic version of the evolutionary process in which no individuals are removed. This idea can be traced back to~\cite{Rabani98}. In the context of evolutionary computation, it has been first used in~\cite{Witt03} (see~\cite{Witt08} for an extended version) in the analysis of a steady-state genetic algorithm with fitness-proportionate selection. In~\cite{JagerskupperW05}, this argument was used in the analysis of a $(\mu+1)$ evolution strategy (in continuous search spaces). Not surprisingly, the analysis of the \mea~\cite{Witt06} uses the artificial populations argument as well.

This technique then found applications in the analysis of memetic algorithms~\cite{Sudholt09}, aging-mechanisms~\cite{JansenZ11tcs}, and non-elitist algorithms~\cite{Lehre10,LehreY12}. The artificial population argument was also used to overcome the difficulties imposed by another removal mechanism, namely Pareto domination in evolutionary multi-objective optimization~\cite{DoerrKV13}. While similar in spirit, this work however uses quite different techniques, e.g., it does not represent the search process via tree structures.

%

Of course, to make the new process really an optimistic version of the original one, we have to ensure that, despite the larger population present, each individual which is also present in the true population has the same power of creating good solutions as in the original process. To ensure this in our process, we assume that in the artificial process each individual creates $\Bin(\lambda,1/\mu)$ offspring. This assumption, in fact, leads to a much more drastic growth of the artificial population than the fact that we disregard selection.

When working with such an artificially enlarged population, there is a risk that the larger population finds it easier to create the optimal solution. This would give weaker lower bounds. So the main art in this proof approach is setting up the arguments in a way that the larger population does still, in an asymptotic sense, not find the optimum earlier than the original process. The reason why this is possible at all is that once selection is disregarded, the process consists only of independent applications of the mutation operator. This allows to use strong-concentration arguments which in the end give the desired result that none of the many members of the artificial population is the optimal solution.

To make this approach formal, we use the following notion of a \emph{complete tree}, which, in simple words,  describes all possible (iterated) offspring which could occur in a run of the evolutionary algorithm. This notion is different from those used in the works above, which all work with certain subtrees of the complete tree and use suitable arguments to reason that the restricted tree still covers all individuals that can, with reasonable probability, appear. We feel that our approach of working in the complete tree is technically simpler. For example, compared to~\cite{Witt06}, we do not first need to argue that with high probability the true tree has only certain depths and then, conditional on this event, argue that it does not contain an optimal solution. Working in the complete tree, we also do not need arguments from branching processes as used in~\cite{LehreY12}.
Of course, the key argument that without selection we only do repeated unguided mutation, is used by us in the same flavor as in all previous works.

More precisely, the complete tree with initial individual $x_0$ is defined recursively as follow. Every vertex is labeled with some individual (a bit-string) which could potentially occur in the evolution process. The labels are not necessarily unique, but every vertex, except the root vertex $v_0$ is uniquely defined by the tuple $(v, t, i)$, where $v$ is the parent vertex (that is either the root vertex, or another vertex defined by a tuple), $t \in \mathbb{N}$ is the iteration when this vertex was created and $i \in [1..\lambda]$ is the number of the vertex among the vertices with the same $v$ and $t$. The tree $T_0 = (V_0,E_0)$ at time $t=0$ consists of the single (root) vertex $v_0$ that is labeled with the bit-string $c(v_0) = x_0$.
Hence $E_0 = \emptyset$. If $T_t = (V_t,E_t)$ is defined for some $t \ge 0$, then we define the tree $T_{t+1} = (V_{t+1},E_{t+1})$ as follows. For each vertex in $V_t$, we add $\lambda$ vertices, connect them to this vertex, and generate their labels via standard-bit mutation from the parent. More precisely, let $N_{t+1} \coloneqq \{(v_t,t+1,i) \mid v_t \in V_t, i \in [1..\lambda]\}$ and $V_{t+1} = V_t \cup N_{t+1}$.
We call $v_t$ the parent of $(v_t,t+1,i)$ and $(v_t,t,i)$ the $i$-th child of $v_t$ in iteration $t+1$. We generate the label $c(v_t,t+1,i)$ by applying standard-bit mutation to $c(v_t)$. We connect each new vertex with its parent, that is, we define $E_{t+1} = E_t \cup \{(v_t,(v_t,t+1,i)) \mid v_t \in V_t, i \in [1..\lambda]\}$. A simple example of a complete tree structure is shown in Figure~\ref{fig:complete_tree}.

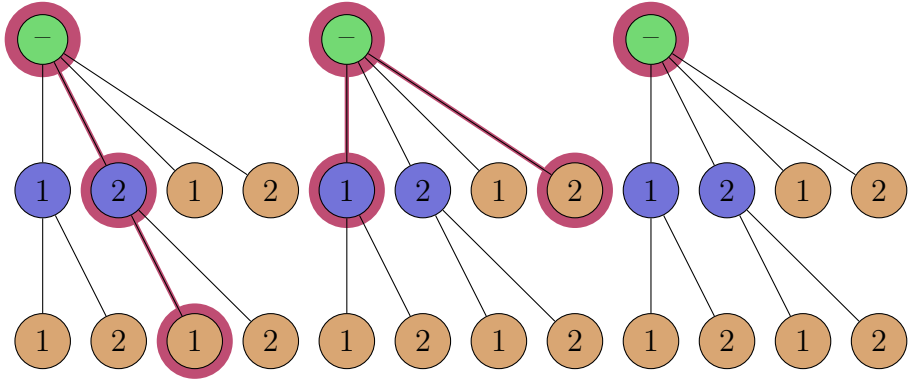
\begin{figure}[t]
  \centering
    \begin{tikzpicture}
      \fill[lightgray!40!purple] (0,0) circle (0.5);
      \fill[lightgray!40!purple] (4,0) circle (0.5);
      \fill[lightgray!40!purple] (8,0) circle (0.5);
      \fill[lightgray!40!purple] (1,-2) circle (0.5);
      \fill[lightgray!40!purple] (4,-2) circle (0.5);
      \fill[lightgray!40!purple] (2,-4) circle (0.5);
      \fill[lightgray!40!purple] (7,-2) circle (0.5);
      \draw[ultra thick, lightgray!40!purple] (0,0) -- (1,-2);
      \draw[ultra thick, lightgray!40!purple] (4,0) -- (4,-2);
      \draw[ultra thick, lightgray!40!purple] (1,-2) -- (2,-4);
      \draw[ultra thick, lightgray!40!purple] (4,0) -- (7,-2);

      \node (v1) [circle, fill=lightgray!60!green, draw] at (0, 0) {--};
      \node (v2) [circle, fill=lightgray!60!green, draw] at (4, 0) {--};
      \node (v3) [circle, fill=lightgray!60!green, draw] at (8, 0) {--};

      \node (v4) [circle, fill=lightgray!60!blue, draw] at (0, -2) {$1$};
      \node (v5) [circle, fill=lightgray!60!blue, draw] at (1, -2) {$2$};
      \node (v6) [circle, fill=lightgray!60!blue, draw] at (4, -2) {$1$};
      \node (v7) [circle, fill=lightgray!60!blue, draw] at (5, -2) {$2$};
      \node (v8) [circle, fill=lightgray!60!blue, draw] at (8, -2) {$1$};
      \node (v9) [circle, fill=lightgray!60!blue, draw] at (9, -2) {$2$};
      \draw (v1) -- (v4);
      \draw (v1) -- (v5);
      \draw (v2) -- (v6);
      \draw (v2) -- (v7);
      \draw (v3) -- (v8);
      \draw (v3) -- (v9);

      \node (v10) [circle, fill=lightgray!60!orange, draw] at (2, -2) {$1$};
      \node (v11) [circle, fill=lightgray!60!orange, draw] at (3, -2) {$2$};
      \node (v12) [circle, fill=lightgray!60!orange, draw] at (6, -2) {$1$};
      \node (v13) [circle, fill=lightgray!60!orange, draw] at (7, -2) {$2$};
      \node (v14) [circle, fill=lightgray!60!orange, draw] at (10, -2) {$1$};
      \node (v15) [circle, fill=lightgray!60!orange, draw] at (11, -2) {$2$};
      \node (v16) [circle, fill=lightgray!60!orange, draw] at (0, -4) {$1$};
      \node (v17) [circle, fill=lightgray!60!orange, draw] at (1, -4) {$2$};
      \node (v18) [circle, fill=lightgray!60!orange, draw] at (2, -4) {$1$};
      \node (v19) [circle, fill=lightgray!60!orange, draw] at (3, -4) {$2$};
      \node (v20) [circle, fill=lightgray!60!orange, draw] at (4, -4) {$1$};
      \node (v21) [circle, fill=lightgray!60!orange, draw] at (5, -4) {$2$};
      \node (v22) [circle, fill=lightgray!60!orange, draw] at (6, -4) {$1$};
      \node (v23) [circle, fill=lightgray!60!orange, draw] at (7, -4) {$2$};
      \node (v24) [circle, fill=lightgray!60!orange, draw] at (8, -4) {$1$};
      \node (v25) [circle, fill=lightgray!60!orange, draw] at (9, -4) {$2$};
      \node (v26) [circle, fill=lightgray!60!orange, draw] at (10, -4) {$1$};
      \node (v27) [circle, fill=lightgray!60!orange, draw] at (11, -4) {$2$};
      \draw (v1) -- (v10);
      \draw (v1) -- (v11);
      \draw (v2) -- (v12);
      \draw (v2) -- (v13);
      \draw (v3) -- (v14);
      \draw (v3) -- (v15);
      \draw (v4) -- (v16);
      \draw (v4) -- (v17);
      \draw (v5) -- (v18);
      \draw (v5) -- (v19);
      \draw (v6) -- (v20);
      \draw (v6) -- (v21);
      \draw (v7) -- (v22);
      \draw (v7) -- (v23);
      \draw (v8) -- (v24);
      \draw (v8) -- (v25);
      \draw (v9) -- (v26);
      \draw (v9) -- (v27);

    \end{tikzpicture}
  \caption{The structure of the complete tree for the $(3 + 2)$~EA after $t = 2$ iterations. The green vertices are the initial vertices, the blue vertices were created in the first iteration and the orange vertices were created in the second iteration. Each vertex is uniquely defined by a tuple of its parent vertex $v$, iteration it was created $t$ and its number $i$ among the children of its parent vertex created at the same iteration (the vertices in the figure are labeled with this number $i$). The highlighted vertices are the ones which were actually created by the algorithm. The labels are omitted in this illustration for reasons of readability}
  \label{fig:complete_tree}
\end{figure}

It is easy to see that a complete tree at time $t$ contains exactly $(\lambda+1)^t$ nodes, since each vertex from $V_t$ has exactly $\lambda$ new children in $V_{t + 1}$. As said earlier, it thus massively overestimates the size of the true population of the EA.

For our purposes, it is not so much the total size of the tree that is important, but rather the number of nodes in a certain distance from the root. We estimate these in the following elementary lemma. Here and in the remainder, by \emph{distance} we mean the graph theoretic distance, that is, the length of the (in this case unique) path between the two vertices. Observe that this can be different from the iteration in which a node was generated. For example, the vertex $(v_0,t,i)$, which is generated in iteration $t$ from the initial vertex, has distance one from $v_0$.
\begin{lemma}\label{ldist}
  Let $T_t$ be a complete tree at time $t$. Let $\ell \in \N_0$. Then $T_t$ contains exactly
  \[\binom{t}{\ell} \lambda^\ell\]
  nodes in distance exactly $\ell$ from the root.
\end{lemma}

\begin{proof}
  If $t < \ell$, then there are no vertices in distance $\ell$ (recall that in our notation $\binom{t}{\ell} = 0$ in this case). Otherwise, let $v$ be a vertex in distance exactly $\ell$ from the root. Then there are times $1 \le t_1 < \dots < t_\ell \le t$ and offspring numbers $i_1, \dots, i_\ell \in [1..\lambda]$ such that with the recursive definition of the vertices $v_1, \dots, v_\ell$ via $v_d = (v_{d-1},t_d,i_d)$ for all $d \in [1..\ell]$, we have $v = v_\ell$.
	Hence, there are at most $\binom{t}{\ell} \lambda^\ell$ vertices in distance $\ell$ from the root. Conversely, each tuple of times and offspring numbers as above defines a different vertex in distance $\ell$. Hence, there are at least $\binom{t}{\ell} \lambda^\ell$ different vertices in distance $\ell$ from the root.
\end{proof}

Since there is no selection in the complete tree, the vertex labels simply arise from repeated mutation. More precisely, a vertex in distance $\ell$ from the root has a label that is obtained from $\ell$ times applying mutation to the root label. This elementary observation allows to estimate the probability that a node label is equal to some target string.

\begin{lemma}\label{lm:prob-opt}
  Consider a complete tree with root label $c(v_0) = x_0$. Let $x^* \in \{0,1\}^n$ with $H(x^*,x_0) \ge n/4$ (where $H$ is the Hamming distance). Let $x$ be the node label of a node in distance $\ell$ from $v_0$. Then
  \[\Pr[x=x^*] \le \min\left\{1, \left(\frac{\ell}{n - 1}\right)^{n/4}\right\} \eqqcolon p(\ell,n).\]
\end{lemma}

\begin{proof}
  The probability that $x = x^*$ is at most the probability that each of the $H(x_0,x^*)$ bits in which $x_0$ and $x^*$ differ was flipped in at least one of the $\ell$ applications of the mutation operator which generated $x$ from $x_0$. For one particular position the probability that this position was involved in one of $\ell$ mutations is $1 - (1 - \frac 1n)^\ell$. For $H(x_0, x^*)$ positions the probability that all of them were involved in one of $\ell$ mutations is
 \[
\left(1-\left(1-\frac{1}{n}\right)^\ell\right)^{H(x_0,x^*)}
 \leq \left(1-\exp\left(-\frac{\ell}{n - 1}\right)\right)^{\frac{n}{4}}
 \leq \left(\frac{\ell}{n - 1}\right)^{\frac{n}{4}},
 \]
 where we used the estimates $(1 - 1/n)^{(n - 1)r} \ge e^{-r}$ valid for all $n \geq 1$ and any positive $r \in \R$, and $e^{-r} \ge 1 - r$ valid for all $r \in \R$.
\end{proof}

We are now ready to prove our lower bound. Since the proof is valid not only for the \onemax function, but for any pseudo-Boolean function with a unique optimum, we formulate the result for such functions. We show extensions to many functions with multiple optima in the following section.

\begin{theorem}
\label{thm:lower}
If $\mu$ is polynomial in $n$, then the $(\mu+\lambda)$ EA with any type of selection of the new parent population (including only selecting from the offspring population) needs an expected number of
\[\Omega\left(\frac{n\log n}{\lambda} + \frac{\mu n}{\lambda}\right)\] iterations to optimize any pseudo-Boolean function with a unique optimum.

If further $\frac{\lambda}{\mu} \geq e^e$, then
the stronger bound
\[\Omega\bigg(\frac{n\log n}{\lambda} +\frac{n\log\log \frac{\lambda}{\mu}}{\log \frac{\lambda}{\mu}}\bigg)\] holds.
\end{theorem}

\begin{proof}
Without any loss of generality in this proof we assume that the function optimized by the algorithm has an optimum in $x^* = (1, \dots, 1)$.

In our proofs we use the following tool. To prove that the expected runtime of the algorithm is $\Omega(f(n))$ for some function $f(n)$, it is enough to prove that the probability that the runtime is less than $f(n)$ is less than some constant $\gamma < 1$, since in this case the expected runtime is not less than $(1 - \gamma)f(n)$.

We first note that the bound $\Omega(\frac{n\log n}{\lambda})$ is easy to prove for the \onemax function. A short, but deep argument for this bound is that the \ea is an unary unbiased black-box complexity algorithm in the sense of Lehre and Witt~\cite{unbiased-bbc-algorithmica}.
Any such algorithm needs an expected number of $\Omega(n \log n)$~\cite{unbiased-bbc-algorithmica} or, more precisely, of at least $e n \ln(n) - O(n)$~\cite{DoerrDY16} fitness evaluations to find the optimum of the \onemax function.

However, we prove the lower bounds for any function with a unique optimum, so we use an elementary argument essentially identical to the one of~\cite{Witt06} as follows. The lower bound $\Omega(\frac{n\log n}{\lambda})$ needs to be shown only in the case $\mu \leq c \log n$, where $c$ is an arbitrarily small constant. For any bit position we have the probability $q_1$ that all individuals in the initial population have a zero-bit in that position that is calculated as
\begin{align*}
q_1 = \left(\frac{1}{2}\right)^\mu \ge \left(\frac{1}{2}\right)^{c\log(n)} = \exp(c\log(n) \log(1/2)) = n^{-c\log(2)}.
\end{align*}
Thus, the expected number $z$ of the bit positions such that all individuals in the initial population have a zero-bit in that position is
\begin{align*}
E[z] = \sum_{i = 1}^n q_1 = n^{1 - c\log(2)}.
\end{align*}
We call such positions \emph{initially wrong positions}. Since the bit values in each bit position and each initial individual are independent, each position is initially wrong or not independently on other positions. 
Hence, by Chernoff bounds (see, e.g., Theorem~1.10.5 in~\cite{Doerr20bookchapter}) the probability $q_2$ that we have at least $E[z]/2 = n^{1 - c\log(2)}/2$ such bit positions is calculated as
\begin{align*}
q_2 = \Pr\left[z \ge (1 - \delta)E[z]\right] \ge 1 - \exp\left( -\frac{\delta^2 E[z]}{2}\right) = 1 - \exp\left( -\frac{n^{1 - c\log(2)}}{8}\right).
\end{align*}

Now we are ready to show that the algorithm does not flip at least one of the bits in the initially wrong positions in $t \coloneqq \lfloor\frac{\alpha(n - 1)\log(n)}{\lambda}\rfloor$ iterations, where $\alpha$ is a constant that will be defined later, with a high (at least $1 - o(1)$) probability. We calculate the probability $q_3$ that one particular bit is flipped at least once in $t$ iterations (or in $\lambda t$ mutations) as
\begin{align*}
	q_3 &= 1 - \left(1 - \frac{1}{n}\right)^{\lambda t} = 1 - \left(1 - \frac{1}{n}\right)^{(n - 1)\frac{\lambda t}{(n - 1)}} \\
	&\le 1 - \exp\left(-\frac{t\lambda}{(n - 1)}\right) \le 1 - e^{-\alpha\log(n)} = 1 - n^{-\alpha}.
\end{align*}
If we have at least $n^{1 - c\log(2)}/2$ initially wrong positions, then the probability $q_4$ that all of them are flipped at least once in $t$ iterations is
\begin{align*}
q_4 &= q_3^{\frac{n^{1 - c\log(2)}}{2}} \le (1 - n^{-\alpha})^{\frac{n^{1 - c\log(2)}}{2}} \\
	  &= (1 - n^{- \alpha})^{n^\alpha \cdot \frac{n^{1 - c\log(2) - \alpha}}{2}} \le \exp\left(-\frac{n^{1 - c\log(2) - \alpha}}{2} \right)
\end{align*}
Thus we have the probability $q_5$ that at least one of the initially wrong bits is not flipped (and thus, the optimum is not found) in $t = \Theta(\frac{n \log(n)}{\lambda})$ iterations at least
\begin{align}\label{eq:prob-nlogn}
\begin{split}
q_5 &\ge q_2 (1 - q_4) \\
		&\ge \left( 1 - \exp\left( -\frac{n^{1 - c\log(2)}}{8}\right) \right) \left(1 - \exp\left(-\frac{n^{1 - c\log(2) - \alpha}}{2} \right)\right) \\
		&\ge 1 - 2\exp\left( -\frac{n^{1 - c\log{2} - \alpha}}{8}\right).
\end{split}
\end{align}

Hence, if $\alpha$ and $c$ satisfy $c\log(2) + \alpha < 1$, (e.g., $\alpha \coloneqq \frac{1}{2}$ and $c \coloneqq \frac{1}{2}$) then the expected runtime of the algorithm is $\Omega(\frac{n \log(n)}{\lambda})$.

To prove the remaining two bounds, we argue as follows. Again using a simple Chernoff bound argument, we first observe that the probability $q_6$ that the number of zero-bits $y$ in the one particular individual in the initial population is less than $n/4$, is estimated as
\begin{align*}
q_6 = \Pr\left[y \le \frac{E[y]}{2}\right] = \Pr\left[y \le (1 - 1/2)E[y]\right] \le \exp\left(-\frac{n}{16}\right).
\end{align*}
Hence, all $\mu$ individuals of the initial population have a Hamming distance of at least $n/4$ from the optimum $x^*$ with probability
\begin{align*}
q_7 = (1 - q_6)^\mu \ge \left(1 - \exp\left(-\frac{n}{16}\right)\right)^\mu \ge \exp\left(-\frac\mu{e^\frac{n}{16} - 1}\right)
\end{align*}
Since $\mu$ is polynomial in $n$, we have $\frac\mu{e^\frac{n}{16} - 1} = o(1)$ and therefore, $q_7 = 1 - o(1)$. Further in this proof we assume that all initial individuals have at least $n/4$ zero-bits.

Clearly, a run of the \ea creates a subforest of $\mu$ disjoint complete trees with random root labels (complete forest).
Whether a node of the complete forest appears in the forest describing the run of the \ea (the forest of the family trees) depends on the node labels (more precisely, on their fitness). However, regardless of the node labels the following is true: If some node $v_s$ is present in the population at iteration $t$, then the edge $(v_s,(v_s,t,i))$ is present in the subforest at most with probability $1/\mu$, because for this it is necessary that the $i$-th offspring generated in iteration $t$ chooses $v_s$ as parent. Consequently, regardless of the nodes labels, the probability that a node in distance $\ell$ from the root in the complete forest enters the population of the \ea, is at most $\mu^{-\ell}$. Since we have not taken into account the node labels, we observe that the probability that a particular node of the complete forest (i) is labeled with the optimum and (ii)~makes it into the population of the \ea, is at most $\mu^{-\ell} p(\ell,n)$ with $p(\ell,n)$ as defined in Lemma~\ref{lm:prob-opt}.

Using a union bound over all nodes in the complete forest up to iteration~$t$, cf.~Lemma~\ref{ldist}, we see that the probability that the \ea finds the optimum within $t$ iterations, is at most
\begin{align}\label{eq:q_opt}
q_{opt} \le \mu \sum_{\ell = 0}^t \binom{t}{\ell} \left(\frac \lambda \mu\right)^\ell p(\ell,n).
\end{align}

Let first $t \coloneqq \lfloor\mu n / 8e\lambda\rfloor$. Using the inequality $\binom{t}{\ell} \le (et / \ell)^\ell$ that follows from Stirling's formula, we estimate the summand $s(\ell) \coloneqq \binom{t}{\ell} (\frac \lambda \mu)^\ell p(\ell,n)$ of $q_{opt}$ for every $\ell \in [0..t]$.
\begin{itemize}
\item By Lemma~\ref{lm:prob-opt} we have $p(\ell,n) \le 1$. Thus, if $\ell \ge n/4$, we estimate
\begin{align*}
s(\ell) = \binom{t}{\ell} \left(\frac \lambda \mu\right)^\ell p(\ell,n) \le \left(\frac{et\lambda}{\ell\mu}\right)^\ell \le \left(\frac{n}{8\ell}\right)^\ell \le (1/2)^\ell \le (1/2)^{n/4}.
\end{align*}
\item By Lemma~\ref{lm:prob-opt} we have $p(\ell,n) \le (\ell/(n - 1))^{n/4}$. Hence, if $n/4 \ge \ell > 0$, we estimate
\begin{align*}
s(\ell) &= \binom{t}{\ell} \left(\frac \lambda \mu\right)^\ell p(\ell,n) \le \left(\frac{et\lambda}{\ell\mu}\right)^\ell \left(\frac{\ell}{n - 1}\right)^{n/4}
\le \left(\frac{n}{8\ell}\right)^\ell \left(\frac{\ell}{n - 1}\right)^{n/4}  \\
&\le \left(\frac{n}{4\ell}\right)^\ell \left(\frac{\ell}{n - 1}\right)^{n/4} \le
\left(\frac{n}{4\ell} \cdot\frac{\ell}{n - 1}\right)^{n/4} \le (1/2)^{n/4}.
\end{align*}
\item Finally, for $\ell = 0$ we have $p(\ell, n) = 0$ and thus $s(\ell) = 0$.
\end{itemize}

Consequently, the optimum is found in less than $t$ iterations if either there is an individual with less than $n/4$ zero-bits in the initial population, or with an exponentially small probability otherwise. Therefore, the probability $q_8$ of finding the optimum in less than $t$ iterations is bounded as
\begin{align}\label{eq:prob-munlambda}
\begin{split}
q_8 &\le (1 - q_7) + q_7 \mu\sum_{\ell = 0}^t s(\ell) \\
    &\le \left(1 - \exp\left(-\frac\mu{e^\frac{n}{16} - 1}\right)\right) + \mu \sum_{\ell = 1}^t (1/2)^{n/4} \\
    &\le \frac\mu{e^\frac{n}{16} - 1} + \frac{\mu^2 n}{8e\lambda} (1/2)^{n/4} = o(1),
\end{split}
\end{align}
since we assumed $\mu$ to be at most polynomial in $n$.

We finish the proof by showing the lower bound $\Omega\left(\frac{n\log\log\frac{\lambda}{\mu}}{\log \frac{\lambda}{\mu}}\right)$ in case when $\frac{\lambda}{\mu} \ge e^e$. For this purpose let  $t = \lfloor\frac{(e - 2)n\ln\ln\frac{\lambda}{\mu}}{4(e + 1)\ln \frac{\lambda}{\mu}}\rfloor$.
Using the complete tree notation we show that the probability that the algorithm finds an optimum in less than $t$ iterations is very small.

For all $\ell \in [0..t]$ consider $s(\ell)$. Using the inequality $\binom{t}{\ell} \le (et / \ell)^\ell$ we estimate the upper bound for it as follows.

\begin{align}\label{eq:sell-last-case}
\begin{split}
s(\ell) &= \binom{t}{\ell} \left(\frac \lambda \mu\right)^\ell p(\ell,n) \le \left(\frac{et\lambda}{\ell \mu}\right)^\ell \left(\frac{\ell}{n - 1}\right)^{n/4} \\
&= \exp\left(\ell \ln \frac{et\lambda}{\ell \mu} + \frac{n}{4} \ln \frac{\ell}{n - 1}\right).
\end{split}
\end{align}

Consider precisely the argument of the exponential function from the last equality in~\eqref{eq:sell-last-case}. For this purpose define $f(\ell) \coloneqq \ell \ln \frac{et\lambda}{\ell \mu} + \frac{n}{4} \ln \frac{\ell}{n - 1}.$
By considering the derivative of $f(\ell)$ on segment $[0, t]$ one can see that it is a monotonically increasing function. Since $t \ge \ell$ and $\frac{\lambda}{\mu} \ge e^e$, we have $\frac{et\lambda}{\ell\mu} \ge e^{e + 1}$ and thus, $\ln\frac{et\lambda}{\ell\mu} \ge e + 1$. Hence,

\begin{align*}
f'(\ell) = \ln \frac{et\lambda}{\ell \mu} - 1 + \frac{n}{4\ell} \ge e + 1 - 1 > 0
\end{align*}

Thus, $f(\ell)$ reaches its maximum when $\ell = t$. Therefore,

\begin{align*}
f(\ell) &\le f(t) \le t \ln\frac{e\lambda}{\mu} + \frac{n}{4}\ln\frac{t}{n - 1} \\
&\le \frac{(e - 2)n \ln\ln\frac\lambda\mu}{4(e + 1)\ln\frac\lambda\mu} \left(\ln\frac\lambda\mu + 1\right) + \frac{n}{4}\ln\frac{(e - 2)n\ln\ln\frac\lambda\mu}{4(e + 1)(n - 1)\ln\frac\lambda\mu} \\
&= \frac{(e - 2)}{4(e + 1)}n\ln\ln\frac{\lambda}{\mu}\left(1 + \frac{1}{\ln\frac\lambda\mu}\right) \\
&+ \frac{n}{4} \left(\ln\ln\ln\frac\lambda\mu - \ln\ln\frac\lambda\mu + \ln\frac{(e - 2)n}{4(e + 1)(n - 1)}\right) \\
&\le \frac{n}{4}\ln\ln\frac\lambda\mu \left( \frac{(e - 2)}{(e + 1)}\left(1 + \frac 1e\right) + \frac{\ln\ln\ln\frac\lambda\mu}{\ln\ln\frac\lambda\mu} - 1 +  \frac{\ln\frac{(e - 2)n}{4(e + 1)(n - 1)}}{\ln\ln\frac\lambda\mu} \right).
\end{align*}

Notice that $\frac{\ln x}{x} \le \frac{1}{e}$ for all $x \ge 1$ and that $\ln\frac{(e - 2)n}{4(e + 1)(n - 1)} < 0$ for all $n > 1$. Therefore we have

\begin{align*}
f(\ell) &\le \frac{n}{4}\ln\ln\frac\lambda\mu \left( \frac{(e - 2)}{(e + 1)}\left(1 + \frac 1e\right) - \left(1 - \frac 1e\right)\right) = -\frac{n\ln\ln\frac\lambda\mu}{4e}.
\end{align*}

Thus, by~\eqref{eq:sell-last-case} we have

\begin{align*}
s(\ell) \le \exp\left(-\frac{n\ln\ln\frac\lambda\mu}{4e} \right) = \left( \ln\frac{\lambda}{\mu} \right)^{-n/4e}.
\end{align*}

By~\eqref{eq:q_opt} summing up $\mu s(\ell)$ for all $\ell \in [0..t]$ we obtain the following upper bound on the probability $q_9$ that the algorithm finds the optimum in less than $t = \Theta(\frac{n\log\log\frac\lambda\mu}{\log\frac\lambda\mu})$ iterations.

\begin{align}\label{eq:prob-nloglog}
\begin{split}
	q_9 &\le (1 - q_7) + q_7 \mu \sum_{\ell = 0}^t s(\ell) \\
	&\le \frac\mu{e^\frac{n}{16} - 1} + \mu \frac{(e - 2)n\ln\ln\frac\lambda\mu}{4(e + 1)\ln\frac\lambda\mu} \left( \ln\frac{\lambda}{\mu} \right)^{-n/4e}.
\end{split}
\end{align}

Notice that $q_9$ is $o(1)$, since we assumed that $\mu$ is polynomial in $n$. Hence, the expected runtime of the algorithm is $\Omega(\frac{n\log\log\frac\lambda\mu}{\log\frac\lambda\mu})$
\end{proof}

\subsection*{Comparison With Other Lower Bounds}

Since all results involved are asymptotically tight, our lower bounds subsume the previous bounds for the \mea and the \lea in the way as discussed for upper bounds in Section~\ref{sec:comparison-upper}.

For general values of $\mu$ and $\lambda$,  the only result~\cite{Qian16} we are aware of proves that for any $\mu$ and $\lambda$ that are at most polynomial in $n$ the runtime of the \ea on every pseudo-boolean function with a unique global optimum is
\begin{align}\label{eq:qian}
\Omega\left(\frac{n\log n}{\lambda} + \frac\mu\lambda + \frac{n\log\log n}{\log n}\right).
\end{align}

By comparing the three terms of this bound with the corresponding terms of our bound \[\Omega\bigg(\frac{n\log n}{\lambda}+\frac{n}{\lambda / \mu} +
\frac{n\log^+\log^+ (\lambda/ \mu)}{\log^+ (\lambda / \mu)}\bigg),\]
we immediately see that our bound is asymptotically at least as large as the one in~\eqref{eq:qian}; note that for the third term, this follows trivially from the assumption that $\lambda$ is polynomial in $n$ and the fact that $x \mapsto \frac{\log\log(x)}{\log(x)}$ is decreasing for $x$ sufficiently large.

There are two cases when our bound is asymptotically greater than~\eqref{eq:qian}.

\noindent\textbf{Setting 1.} Let $\frac{\lambda}{\mu} = O(1)$ and $\mu = \omega(\log(n))$. Then our bound is $\Omega(\frac{n\mu}{\lambda})$, which is at least $\Omega(n)$. On the other hand,~\eqref{eq:qian} is
\begin{align*}
\frac{n\log n}{\lambda} + \frac\mu\lambda + \frac{n\log\log n}{\log n} = \frac{n \, o(\mu)}{\lambda} + \frac{\mu}{\lambda} + o(n) = o{\left(\frac{n\mu}{\lambda}\right)}.
\end{align*}

\noindent\textbf{Setting 2.} Let $\log\frac{\lambda}{\mu} = \omega(\log n)$. This implies that $\frac\lambda\mu = \omega(n)$ and thus
\[
\log n = o{\left(\frac{n\log \log n}{\log n}\right)} = o{\left(\frac{\frac{\lambda}{\mu}\log \log\frac{\lambda}{\mu}}{\log\frac{\lambda}{\mu}}\right)}.
\]
Therefore, we have
\[
\frac{n\log n}{\lambda} = o{\left(\frac{n \log \log\frac{\lambda}{\mu}}{\mu \log\frac{\lambda}{\mu}}\right)} = o{\left(\frac{n \log \log\frac{\lambda}{\mu}}{\log\frac{\lambda}{\mu}}\right)}.
\]
Hence, the lower bound given in Theorem~\ref{thm:lower} simplifies to $\Omega(\frac{n \log \log\frac{\lambda}{\mu}}{\log\frac{\lambda}{\mu}})$.

On the other hand, the bound~\eqref{eq:qian} is of the asymptotically smaller order $o(\log n) + o(1) + O(\frac{n \log \log n}{\log n}) = O(\frac{n \log \log n}{\log n})$.

\section{Extending the Lower Bounds to All Functions Having Not Excessively Many Global Optima}\label{scn:lower-extended}

Since the family tree technique depends little on the particular function to be optimized, Witt~\cite{Witt06} extended his lower bounds for \onemax to a much broader class of functions. He proved that the \mea needs  $\Omega(\mu n)$ iterations to find a global optimum of any function that satisfies one of the following conditions. (i)~The function has at most $2^{o(n)}$ optima. (ii)~All optima have at least $n/2 + \eps n$ one-bits or all optima have at least $n/2 + \eps n$ zero-bits, where $\eps>0$ is an arbitrary constant.

In this section we extend our lower bounds of Section~\ref{scn:lower} to a wide class of functions as well. In particular, we show that Witt's results are valid for all functions with at most $2^{\beta n}$ optima, where $\beta$ is some constant less than $\frac{1}{16\ln 2}$, regardless of the positions of the optima.

To reach our goal we exploit the fact that in Theorem~\ref{thm:lower} we proved very small values for the probabilities that the runtime is less than some threshold (see~\eqref{eq:prob-nlogn}, \eqref{eq:prob-munlambda} and \eqref{eq:prob-nloglog}), while it would have been enough to prove that they are some constants less than one.

\begin{theorem}\label{th:non-unimodal-1}
For any constant $\varepsilon > 0$ there exists another constant $c > 0$ such that if $\mu < c \ln n$, then for any $n$-dimensional pseudo-Boolean function with not more than $2^{n^{1 - \varepsilon}}$ optima the \ea takes at least $\Omega(\frac{n\log n}{\lambda})$ iterations in expectation and with high probability to find an optimum.
\end{theorem}

\begin{proof}
	Let $c$ be some arbitrary small positive constant and let $\mu < c\ln n$. By~\eqref{eq:prob-nlogn} the probability that the algorithm finds a particular optimum in less than $t \coloneqq \frac{\alpha n\log n}{\lambda}$ iterations (where $\alpha$ is some arbitrary constant) is
	\begin{align*}
	1 - q_5 \le 2 \exp\left(-\frac{n^{1 - c\ln2 - \alpha}}{8}\right).
	\end{align*}

	If we have at most $2^{n^{1 - \varepsilon}}$ optima, then by a union bound over all optima we obtain that the probability $q_{10}$ that the algorithm finds an optimum in less than $t$~iterations is
	\begin{align*}
		q_{10} &\le (1 - q_5) 2^{n^{1 - \varepsilon}} \le 2 \exp\left(-\frac{n^{1 - c\ln2 - \alpha}}{8}\right) \exp\left(n^{1 - \varepsilon}\ln 2\right) \\
				&= 2 \exp\left(n^{1 - \varepsilon}\ln 2 -\frac{n^{1 - c\ln2 - \alpha}}{8} \right).
	\end{align*}

	This probability $q_{10}$ tends to zero with growing $n$ if and only if the argument of the exponential function tends to negative infinity. It does so if and only if $\alpha$ and $c$ satisfy	$\alpha + c\ln 2 < \varepsilon$. Since $\varepsilon$ is a positive constant, we can choose $\alpha \coloneqq \varepsilon/2$ and $c \coloneqq \varepsilon/2$ to satisfy this condition.
\end{proof}

The actual reason that the algorithm cannot find an optimum faster than in $\Omega(\frac{n\log n}{\lambda})$ iterations is the coupon collector effect when the algorithm tries to flip the few wrong bits left in the end of the optimization. However, if we have $2^{\Theta(n)}$ optima, the algorithm avoids this effect. To illustrate this idea consider the \oea that optimizes the \onemax function, but the bit-strings with less than $cn$ zero-bits, where $c$ is some small constant, are considered optimal. Thus, this functions has no more than $O(2^{c\log_2(1/c)n}) \subseteq 2^{\Theta(n)}$ optima. Clearly, the runtime of the \oea on such function is linear, which may be proven with simple additive drift argument.


The following two theorems extend our $\Omega(\frac{n\mu}{\lambda})$ and $\Omega(\frac{n\log\log\frac\lambda\mu}{\log\frac\lambda\mu})$ bounds to the functions with $2^{O(n)}$ optima.

\begin{theorem}\label{th:non-unimodal-2}
If $\mu$ is at most polynomial in $n$, then the \ea optimizes any pseudo-Boolean function with at most $2^{\beta n}$ optima, where $\beta$ is some constant less than $\frac{1}{16\ln 2}$, in $\Omega(\frac{\mu n}{\lambda})$ iterations. If $\frac\lambda\mu > e^e$, then the stronger bound $\Omega(\frac{n\log\log\frac\lambda\mu}{\log\frac\lambda\mu})$ holds.
\end{theorem}

\begin{proof}
By~\eqref{eq:prob-munlambda} the probability that the algorithm finds a particular optimum in less than $t \coloneqq \lfloor \frac{\mu n}{8e\lambda} \rfloor$ iterations is
\begin{align*}
q_8 \le \frac\mu{e^\frac{n}{16} - 1} + \frac{\mu^2 n}{8e\lambda} \left(\frac 12 \right)^\frac{n}{4}.
\end{align*}
By a union bound taken over no more than $2^{\beta n}$ optima, the probability $q_{11}$ that the algorithm finds any optimum in this time is
\begin{align*}
q_{11} \le q_8 2^{\beta n} \le \frac{\mu e^{(\ln2)\beta n - \frac{n}{16}}}{1 - e^{-\frac{n}{16}}} + \frac{\mu^2 n}{8e\lambda} 2^{\beta n - \frac{n}{4}}.
\end{align*}

Since $\beta < \frac{1}{16\ln2}$ and $\beta$ is a constant, we have both $(\ln2)\beta n - \frac{n}{16} < 0$ and $\beta n - \frac{n}{4} < 0$ (and both of them are linear in $n$). Thus, $q_{11}$ tends to zero with growing~$n$. Hence, the expected runtime of the algorithm is $\Omega(t) = \Omega(\frac{\mu n}{\lambda})$.

To prove the $\Omega(\frac{n\log\log\frac\lambda\mu}{\log\frac\lambda\mu})$ bound we argue in a similar way.
By~\eqref{eq:prob-nloglog} the probability that the algorithm finds a particular optimum in less than $t \coloneqq \lfloor\frac{(e - 2)n\ln\ln\frac{\lambda}{\mu}}{4(e + 1)\ln \frac{\lambda}{\mu}}\rfloor$ iterations is
\begin{align*}
q_9 \le \frac\mu{e^\frac{n}{16} - 1} + \mu \frac{(e - 2)n\ln\ln\frac\lambda\mu}{4(e + 1)\ln\frac\lambda\mu} \left( \ln\frac{\lambda}{\mu} \right)^{-n/4e}.
\end{align*}

By a union bound taken over no more than $2^{\beta n}$ optima, the probability $q_{12}$ that the algorithm finds any optimum in this time is
\begin{align*}
q_{12} \le q_9 2^{\beta n} \le \frac{\mu e^{\beta n\ln2 -n/16}}{1 - e^{-\frac{n}{16}}} + \mu \frac{(e - 2)n\ln\ln\frac\lambda\mu}{4(e + 1)\ln\frac\lambda\mu} e^{\beta n\ln2 -n/4}.
\end{align*}

Since $\beta  < \frac{1}{16\ln2}$ and $\beta$ is a constant, we have both $(\ln2)\beta n - \frac{n}{16} < 0$ and $\beta n\ln2 - \frac{n}{4} < 0$ (and both of them are linear in $n$). Thus, $q_{12}$ tends to zero with growing~$n$. Hence, the expected runtime of the algorithm is $\Omega(t) = \Omega(\frac{n\log\log\frac\lambda\mu}{\log\frac\lambda\mu})$.
\end{proof}

\section{Analysis of the \protect\llea}
\label{sec:llea}

In this section we prove that our results (both upper bound from Theorem~\ref{thm:general} and lower bound from Theorem~\ref{thm:lower}) hold in an analogous fashion also for the \llea, that is, we show that this algorithm optimizes \onemax in an expected number of $\Theta(\frac{n \log n}{\lambda} + n)$ iterations. This improves over the $O(\frac{n \log n}{\lambda} + n \log \lambda)$ proven bound and the $O(\frac{n \log n}{\lambda} + n \log\log n)$ conjecture of~\cite{Chen09}.

\begin{algorithm2e}[t]%

	\underline{\textbf{Initialization:}}\\
	Create a population of $\lambda$ individuals by choosing $x^{(i)} \in \{0,1\}^n$, $1\leq i \leq \lambda$ uniformly at random. Let the multiset $X^{(0)} := \{x^{(1)}, ..., x^{(\lambda)}\}$ be the population at time 0. Let $t := 0$.

 \underline{\textbf{Optimization:}}\\
\While{an optimum has not been reached}{
$X' := X^{(t)}$\;
{\textbf{Mutation phase:}}\\
\For{$i=1, \ldots, \lambda$\label{line:mutstart}}{
$x := $ the $i$-th individual from $X^{(t)}$ (deterministic selection)\;
Create $x'$ by
flipping each bit of $x$ with probability $p$\;
$X' := X' \cup \{x'\}$\;
}
{\textbf{Selection phase:}}

Create the multiset $X^{(t+1)}$, the population at time $t+1$, by deleting the $\lambda$ individuals with lowest
$f$-value in $X'$\;
$t:=t+1$\;
}

\caption{The \protect\llea, maximizing a given function $f : \{0,1\}^n \to \R$, with population size~$\lambda$ and mutation rate $p$.}
\label{alg:llea}
\end{algorithm2e}

Due to the differences in the algorithms, to prove our results we obviously cannot just apply the previous theorems in this work to the case $\lambda = \mu$. We recall that the \llea uses a different parent selection. While the classic \ea chooses each parent independently and uniformly at random from the $\mu$ individuals, the \llea creates exactly one offspring from each parent. The pseudocode of the \llea is shown in Algorithm~\ref{alg:llea}. We note that~\cite{Chen09} also use a slightly different way of selecting the next parent population. In principle, they take as new parent population the $\mu$ best individuals among parents and offspring (plus-selection). If this would lead to a new parent population only consisting of offspring, they remove the weakest offspring and replace it with the strongest individual from the previous parent population. Since this appears to be a not very common way of selecting the new population, we shall work with the classic plus-selection, favoring offspring in case of ties, and breaking further ties randomly (though, indeed, the tie-breaking is not important when optimizing \onemax via unary unbiased black-box algorithms). We note without proof that the following results and proofs are valid for the precise algorithm regarded in~\cite{Chen09} as well.

We start by proving the upper bound for the runtime.

\begin{theorem}
The expected runtime of the \llea on the \onemax function is $O(\frac{n\log n}{\lambda} + n)$.
\end{theorem}

\begin{proof}
  We aim at adapting Theorem~\ref{thm:general} for the \llea. For this purpose we note that the proof of Theorem~\ref{thm:general} only depends on the expected level improvement times $E[\tilde T_i]$ computed in Corollary~\ref{lm:tilde-t-general}, which again depend on the times needed for increasing the number of fit individuals computed in Lemma~\ref{lm:mu-best-general}. Therefore, it suffices to show that the estimates of Lemma~\ref{lm:mu-best-general} and Corollary~\ref{lm:tilde-t-general} are also valid for the \llea.

  We prove that Lemma~\ref{lm:mu-best-general} holds for the \llea by observing that the probability $p_2(j)$ to create at least one copy of the fit individual satisfies the same estimate as the one used for the \ea, which is~\eqref{eq:p2}, with $\mu =\lambda$. For the \llea, $p_2(j)$ is at least the probability that at least one of the $j$ fit parent individuals creates as offspring a copy of it. By Lemma~\ref{lm:bernoulli} we have
  \begin{align*}
    p_2(j) \ge 1 - \left(1 - \left(1 - \frac 1n\right)^n\right)^j \ge 1 - \left(1 - \frac{1}{2e}\right)^j \ge \frac{1}{1 + \frac{2e}{j}},
  \end{align*}
  which is the same estimate as for the \ea (with $\mu = \lambda$).

  To prove that Corollary~\ref{lm:tilde-t-general} holds for the \llea as well, it is sufficient to show that the probability $p''(i)$ to create a superior individual satisfies as well the estimate~\eqref{eq:p-doble-dash} in the case $\mu=\lambda$. The probability $p''(i)$ is at least the probability that for at least one of the $\mu_0(i)$ best individuals the offspring is better than its parent. Using Lemma~\ref{lm:bernoulli} we calculate
	\begin{align*}
		p''(i) &\ge 1 - \left(1 - \frac{n - i}{n}\left(1 - \frac{1}{n}\right)^{n- 1} \right)^{\mu_0(i)} \ge 1 - \left(1 - \frac{n - i}{en}\right)^{\mu_0(i)} \\
		& \ge 1 - \frac{1}{1 + \frac{\mu_0(i)(n - i)}{ne}},
	\end{align*}
  which is the same value as in  Corollary~\ref{lm:tilde-t-general} when $\mu = \lambda$.
\end{proof}

Comparing this bound with the bound $O(\frac{n\log n}{\lambda} + n\log\lambda)$ proven in~\cite{Chen09} and the bound $O(\frac{n\log n}{\lambda} + n\log\log n)$ conjectured in the same work, we immediately see that ours is at least as strong as these two for all values of $\lambda$. For $\lambda = \omega(\frac{\log n}{\log\log n})$, our bound is asymptotically smaller than both the proven bound and the conjecture.


We now prove a matching lower bound, which agrees with the one of Theorem~\ref{thm:lower} in the case of $\mu = \lambda$.

\begin{theorem}
  If $\lambda$ is polynomial in $n$ then the expected runtime of the \llea on the \onemax function is $\Omega(\frac{n\log n}{\lambda} + n)$.
\end{theorem}

\begin{proof}
  We show that the main arguments of the proof for this bound in Theorem~\ref{thm:lower} are also valid for this parent selection mechanism.

  To prove the $\Omega(\frac{n\log n}{\lambda})$ bound we can repeat the arguments from Theorem~\ref{thm:lower} without any changes. One needs to prove this bound only for $\lambda < c\log n$ for some arbitrary small constant $c$. The main argument is that with high probability there is a set of bits
  which were in a wrong position in all initial individuals and that at least one of those bits was not flipped by any of $t\lambda$ applications of the mutation operator for some $t = \Theta(\frac{n\log n}{\lambda})$. This argument stays valid for the fair parent selection as well.

  To prove the $\Omega(n)$ bound we consider the complete trees for the \llea. Since in a run of the \llea each individual in the population creates exactly one offspring, the complete trees now have a slightly different structure, namely each node of the tree has exactly one child at each time step (instead of $\lambda$ children). In return, we cannot argue that each edge is present in the true family tree with probability at most $1/\mu$ only (so we assume that all these edges are in fact present). Since $\lambda = \mu$, these two effects cancel.



  More precisely, following the proof of Theorem~\ref{thm:lower} we argue that with high probability $q_7 \ge \exp(-\frac\lambda{e^\frac{n}{16} - 1})$ all initial individuals have at least $n/4$ wrong bits. Next, we argue that in an analogous fashion as in~\eqref{eq:q_opt} -- and this is where the two effects truly cancel -- the probability $q_{opt}$ that the optimum occurs in any tree in less than $t \coloneqq \lceil\frac{n}{8e}\rceil$ iterations is at most
  \begin{align*}
  q_{opt} \le \lambda \sum_{\ell = 0}^t \binom{t}{\ell} p(\ell,n) \le \lambda t \left(\frac{1}{2}\right)^{n/4}.
  \end{align*}

  Since we only consider $\lambda$ that is polynomial in $n$, this entity tends to zero, when $n$ tends to infinity. Therefore, the probability that the algorithm finds an optimum in $t = \Theta(n)$ iterations is at most $(1 - q_7) + q_7 q_{opt}$ that is less than some constant, if $n$ is large enough. Hence, the expected runtime of the \llea is $\Omega(n)$.
\end{proof}

%
%

\section{Discussion and Conclusion}
\label{scn:discussions}

In this work, we determined -- tight apart from constant factors -- the runtime of the \ea on the \onemax benchmark problem. This is thus one of the few tight runtime analyses taking into account more than a single parameter (\cite{GiessenW17,DoerrD18} are the other two such works we are aware of).

Not surprisingly for a simple function like \onemax, our result does not indicate that it is advantageous to use larger parent or offspring populations. Indeed, it follows from~\cite[Theorem~6.2]{Witt13} (see~\cite{Doerr19tcs} for a simplified proof) that for any $\mu$ and $\lambda$ the runtime of the \ea stochastically dominates the runtime of the \oea with best-of-$\mu$ initialization. The runtime difference between the \oea with  best-of-$\mu$ initialization and with the usual random initialization is small, roughly an additive $\Theta(\sqrt{n \ln \mu})$ term~\cite{LaillevaultDD15}.

While our result does not show an advantage of using larger populations, it does show that using moderate-size populations is not overly costly. For example, as long as $\mu, \lambda = O(\log n)$, the \ea takes $\Theta(n \log n)$ fitness evaluations to find the optimum. This observation could indicate that using such population sizes is generally an interesting idea -- we could speculate that there is no harm from using such populations, but there could be other advantages.

In the light of recent other work, our work suggests two directions for further research. In~\cite{GiessenW17}, a precise runtime analysis for the \lea with general mutation rate $c/n$, $c$ a constant, on the \onemax benchmark was conducted. It suggests that the precise mutation rate is important when $\lambda$ is small, but less decisive when $\lambda$ is large. It would be interesting to know to what extent this result carries over to the \ea. In~\cite{BadkobehLS14,DoerrGWY19,DoerrWY18}, it was shown that various dynamic choices of the mutation rate can reduce the runtime of the \lea on \onemax. Again, it would be interesting to see to what extend a similar behavior is true for the \ea.

\subsection*{Acknowledgement}

We are thankful to Jiefeng Fang and Tangi Hetet for their contributions to the preliminary version~\cite{AntipovDFH18} of this work. The second author would like to thank Jon Rowe for pointing him to the arguments used in~\cite{RoweD14}, which  were used in the proof of Lemma~\ref{lm:bernoulli}. The first author was supported by the Government of Russian Federation (Grant 08-08). 


\newcommand{\etalchar}[1]{$^{#1}$}

}
\end{document}